\documentclass[letterpaper]{article} % DO NOT CHANGE THIS
\usepackage{aaai25}  % DO NOT CHANGE THIS
\usepackage{times}  % DO NOT CHANGE THIS
\usepackage{helvet}  % DO NOT CHANGE THIS
\usepackage{courier}  % DO NOT CHANGE THIS
\usepackage[hyphens]{url}  % DO NOT CHANGE THIS
\usepackage{graphicx} % DO NOT CHANGE THIS
\usepackage{natbib}  % DO NOT CHANGE THIS AND DO NOT ADD ANY OPTIONS TO IT
\usepackage{caption} % DO NOT CHANGE THIS AND DO NOT ADD ANY OPTIONS TO IT
\usepackage{algorithm}
\usepackage{algorithmic}

\usepackage[utf8]{inputenc} % allow utf-8 input
\usepackage[T1]{fontenc}    % use 8-bit T1 fonts
\usepackage{array}
\usepackage{url}            % simple URL typesetting
\usepackage{booktabs}       % professional-quality tables
\usepackage{amsfonts}       % blackboard math symbols
\usepackage{nicefrac}       % compact symbols for 1/2, etc.
\usepackage{microtype}      % microtypography
\usepackage{xcolor}         % colors
\usepackage{graphicx}
\usepackage{amsthm}
\usepackage{amsmath}
\usepackage{algorithmic}
\usepackage{algorithm}
\newtheorem{theorem}{Theorem}
\newcommand{\by}{\boldsymbol{y}}

\def\by{\mathbf{y}}
\def\CFP{1}   %{C_{\text{FP}}}
\def\CNN{C_{Med}}    %{\text{NN}}}
\def\CMN{C_{Lg}}         %{C_{\text{MN}}}
\usepackage{newfloat}
\usepackage{listings}
\DeclareCaptionStyle{ruled}{labelfont=normalfont,labelsep=colon,strut=off} % DO NOT CHANGE THIS
\floatstyle{ruled}
\newfloat{listing}{tb}{lst}{}
\floatname{listing}{Listing}
\title{Dimension Reduction with Locally Adjusted Graphs}
\author {
    Yingfan Wang\equalcontrib, Yiyang Sun\equalcontrib, Haiyang Huang\equalcontrib, Cynthia Rudin
}
\affiliations {
    \textsuperscript{\rm }Duke University\\
    \texttt{\{yingfan.wang, yiyang.sun, haiyang.huang, cynthia.rudin\}@duke.edu}
}

\begin{document}

\maketitle

\begin{abstract}
Dimension reduction (DR) algorithms have proven to be extremely useful for gaining insight into large-scale high-dimensional datasets, particularly finding clusters in transcriptomic data. The initial phase of these DR methods often involves converting the original high-dimensional data into a graph. In this graph, each edge represents the similarity or dissimilarity between pairs of data points. However, this graph is frequently suboptimal due to unreliable high-dimensional distances and the limited information extracted from the high-dimensional data. This problem is exacerbated as the dataset size increases. 
If we reduce the size of the dataset by selecting points for a specific sections of the embeddings, the clusters observed through DR are more separable since the extracted subgraphs are more reliable. 
%We have observed that some DR methods inadvertently merge multiple clusters that could be distinguished if the algorithm were applied only to specific sections of the embedding. 
In this paper, we introduce LocalMAP, a new dimensionality reduction algorithm that dynamically and locally adjusts the graph to address this challenge. By dynamically extracting subgraphs and updating the graph on-the-fly, LocalMAP is capable of identifying and separating real clusters within the data that other DR methods may overlook or combine. We demonstrate the benefits of LocalMAP through a case study on biological datasets, highlighting its utility in helping users more accurately identify clusters for real-world problems.

\begin{links}
    \link{Code}{https://github.com/williamsyy/LocalMAP}
\end{links}
\end{abstract}

\section{Introduction}

Dimension reduction (DR) is an important data visualization strategy for understanding the structure of complicated high-dimensional datasets. DR is used extensively in image, text, and biomedical datasets, particularly -omics \cite{cao2019single,becht2019dimensionality, amezquita2020orchestrating, dries2021giotto, atitey2024model, bohm2022unsupervised, mu2017all,raunak2019effective}.
Beginning in the original high-dimensional space, DR methods typically first abstract the data into a \textit{graph}, with each edge representing either similarity (positive edge) or dissimilarity (negative edge). This graph is then used to optimize the low-dimensional embedding by applying attractive forces along the positive edges and repulsive forces along the negative edges. Clearly, the quality of the graph, which connects the original high-dimensional data to the low-dimensional embedding, is critical to this process.  However, given the complex nature of high-dimensional data, it is challenging to construct a graph that accurately represents all patterns and structures within the data, and graphs are usually sub-optimal.

In this work, we provide two key insights as to why the graphs could be flawed. First, we found that high-dimensional distances become less informative due to the curse of dimensionality, where measurements of similarity and dissimilarity determined using high-dimensional distances do not necessarily reflect similarities and dissimilarities along the data manifold. This issue is more pronounced in DR methods that select a small group of nearest neighbors (NNs) to form positive edges of the graph and apply strong attractive forces along these edges. When a ``false'' positive edge is constructed based on an inaccurate similarity measure between a pair of points that are actually dissimilar, strong attractive forces will pull them closer. If there are many such false positive edges, the DR method will generate overlapping clusters without clear boundaries, even if the data have distinct clusters. Since DR methods are unsupervised, the user would not be able to determine that distinct clusters exist -- they would be blurred together in the DR result. As we show in this work, eliminating these false positive edges can dramatically improve DR projections.

Our second insight is that missing edges (lack of negative ``further pair'' edges between far points in high-dimensional space) also contribute to unwanted overlapping of clusters, since such edges help to define boundaries between clusters. However, as the scale of the dataset increases, the set of negative edges become both insufficient and less effective, as discussed in Section \ref{sec:insufficient_FP}. Some of the most important applications of DR (single cell, transcriptomics) generate large high-dimensional datasets with many clusters, so it is important that there are enough further pair (FP) edges to separate them.  %Therefore, improving the quality of DR embeddings for large high-dimensional datasets is beneficial.

Based on the two insights discussed above, we propose LocalMAP, a new DR algorithm that dynamically and locally adjusts the graph \textit{during dimension reduction} to address the aforementioned issues with graph construction. LocalMAP \textit{detects false positive edges and removes them}, using ideas similar to outlier detection in robust statistics. This allows clusters to separate. It also \textit{adds more further pair edges} dynamically, allowing crisper cluster boundaries to appear. Together, these ideas within LocalMAP produce clear, crisp separated clusters where other methods fail.
%It is special because it adjusts the graph during the optimization (instead of solely depending on pre-defined graph which depends solely on untrustworthy high-dim distance) . Also it identifies and address the issue caused by increasing dataset size which is unique.

Figure \ref{fig:blurred_DR} shows the results of several high-quality DR approaches, including t-SNE \cite{van_der_Maaten08}, UMAP \cite{UMAP} and PaCMAP \cite{pacmap}, on the MNIST dataset where there are 10 separated clusters in the high-dimensional space. Algorithms t-SNE, UMAP and PaCMAP generate DR embeddings without clear boundaries between clusters, while LocalMAP generates a high-quality DR embedding with well-defined boundaries that are visible even without class information provided to the algorithm.

\begin{figure}[ht]%\begin{figure}[htb]
%\begin{center}
\centerline{\includegraphics[width=1.2\columnwidth]
{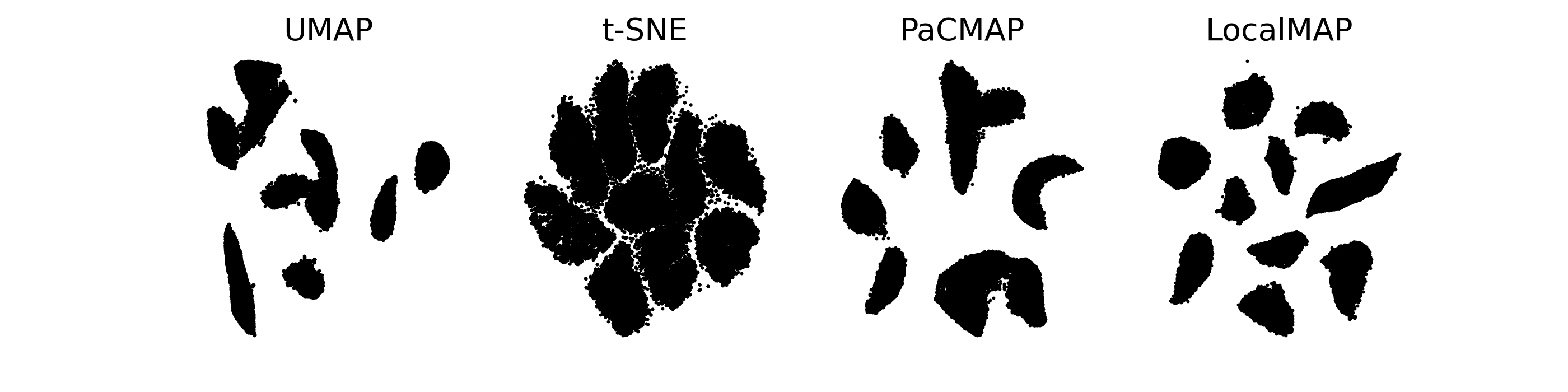}}
\caption{DR embeddings on MNIST dataset which contains $10$ digit classes. Our LocalMAP method is on the right. The colored embeddings with true labels are shown in Figure \ref{fig:case_study}.}
\label{fig:blurred_DR}
%\end{center}
\end{figure}

In this work, we provide more evidence for the insights discussed above, in terms of both theory based on simple clustering models and empirical evidence. We also provide LocalMAP and a comprehensive evaluation of it, based on \citet{huang2022towards}. 
%However, evaluations of DR in the literature do not take into account the ability to separate clusters with clear boundaries very well, so we introduce a new evaluation metric: the silhouette score, which is commonly used for evaluating the quality of clustering in unsupervised learning but not for DR. The silhouette score is able to capture how well the clusters can be identified and clearly separated that other methods cannot. 
We provide case studies showing that LocalMAP is able to find true clusters -- and not false clusters -- more reliably than other DR approaches. 

\section{Related Work}
\label{sec:related_work}

While DR methods that preserve \textit{global structure} date back to 1901 with PCA \cite{Pearson01}, multidimensional scaling \cite{Torgerson52}, and non-negative matrix factorization \cite{lee1999learning}, methods that preserve \textit{local structure} have become indispensable and ubiquitous in numerous applications, particularly in -omics research. Local methods include Isomap \cite{Tenenbaum00}, Local Linear Embedding (LLE) \cite{Roweis00}, Laplacian Eigenmap \cite{Belkin01}, and more recent Neighborhood Embedding (NE) algorithms like t-SNE \cite{van_der_Maaten08}, LargeVis \cite{Tang16}, UMAP \cite{UMAP}, PaCMAP \cite{pacmap}, TriMAP \cite{TriMAP}, NCVis \cite{artemenkov2020ncvis}, h-NNE \cite{sarfraz2022hnne}, Neg-t-SNE \cite{damrich2023from} and InfoNC-t-SNE \cite{damrich2023from} and many others \cite{van2022probabilistic, zu2022spacemap}. Global methods typically are not able to preserve separations between clusters or clearly display the data manifold, whereas local methods can sometimes do so; thus, they provide a unique perspective on the data that is difficult to gain in any other way.

Because DR methods are unsupervised, there is no common objective function like there would be for supervised learning (e.g., classification error). However, there are principles that reliable DR loss functions typically obey \cite{pacmap}. We will discuss those in Appendix C, as LocalMAP obeys these principles based on what it inherits from PaCMAP.

There are works that discuss how DR methods' behavior is affected by different components of the DR algorithm. For example, there are several papers discussing the challenges of tuning parameters and applying these methods in practice 
\cite{wattenberg2016use,cao2017automatic,nguyen2019ten,Belkina19OptSNE, kobak2021initialization}. \citet{pacmap, Unify} also discuss how changing different graph components and loss functions affect DR methods' behavior. 

\textit{None of the above approaches adjusts the graph itself during DR optimization.} The graph is typically considered to be fixed as ground truth; the graph information is analogous to the labels for classification tasks that are also considered ground truth. However, recent work has found improvements in classification performance by identifying possibly mislabeled points and omitting them \cite{INCV2019,coteaching, northcutt2021pervasive}, showing that there is value in identifying and removing labels that appear to be wrong \textit{during the classification process}. There is also a field of robust statistics that identifies and omits outliers that would wreck performance, and aims to ensure results are robust to assumptions on the data distribution \citep{MartinRobustStatsBook2018,li2020dividemix}. Our work is thus unique in extending these types of idea to DR, and not trusting every graph element as if it were correct. As discussed, we know the graph is probably wrong when the Euclidean distance is not the same as the distance along the data manifold (geodesic distance).

\section{Review of PaCMAP}
The proposed LocalMAP algorithm starts from a (faulty) embedding that is already formed. We first review PaCMAP since LocalMAP can start after its first two phases.
Consider a dataset \( \mathbf{X} = \{\mathbf{x}_1, \mathbf{x}_2, \ldots, \mathbf{x}_n\} \) consisting of \( n \) points in a high-dimensional space. PaCMAP seeks to find a low-dimensional embedding \( \mathbf{Y} = \{\mathbf{y}_1, \mathbf{y}_2, \ldots, \mathbf{y}_n\} \), where each \( \mathbf{y}_i \) corresponds to \( \mathbf{x}_i \). PaCMAP first constructs a graph in high-dimensional space with three kinds of pairs: NN pairs (near neighbors in the high-dimensional space), MN pairs (mid-near pairs, close but not as close as neighbors), and FP pairs (further point pairs, far in the high-dimensional space). Optimization of the low-dimensional embedding $\mathbf{Y}$ is performed using a simple objective:
\begin{equation}
\begin{aligned}
    &Loss^{\text{PaCMAP}} = w_{\text{NN}}\cdot\sum_{(i,j) \text{: NN}}\frac{\tilde{d}_{ij}}{\CNN + \tilde{d}_{ij}} + \\ &w_{\text{MN}}\cdot\sum_{(i,k) \text{: MN}}\frac{\tilde{d}_{ik}}{\CMN+ \tilde{d}_{ik}}
     + w_{\text{FP}}\cdot\sum_{(i,l) \text{: FP}}\frac{1}{\CFP + \tilde{d}_{il}}
\end{aligned}
\end{equation}
where $\tilde{d}_{ij}=d_{ij}^2+1=\|\by_{i} - \by_{j}\|^2 + 1$. $\CMN$ and $\CNN$
%, and \CFP\CFP 
are nontunable parameters controlling the scale of the embedding, set at $\sim$10K and $\sim$10. The weights $w_{\text{NN}}$, $w_{\text{MN}}$, and $w_{\text{FP}}$ are adjusted to balance attraction and repulsion.
%starting with large wMNw_{\text{MN}} decreasing to 0, wNNw_{\text{NN}} starts 0 and gets larger, and wFPw_{\text{FP}} gets larger also. 
Neither the weights nor the parameters should be modified by users; they perform well across datasets \citep{pacmap}.

\section{Dynamically and Locally Adjusted Graph}

\subsection{Insight 1: False positive nearest neighbor edges pull nearby clusters together, losing clear boundaries between them}

Let us consider an experiment to illustrate this. DR methods apply attractive forces along the NN edges (between points that are close in the high-dimensional space) and apply repulsive forces along FP edges (points that are far in the high-dimensional space). %We will consider only the NN and FP edges, since forces along MN edges are small and quickly fade to 0.
%Considering (1.) the magnitude of attractive forces on MN edges is much smaller than that on NN edges, and (2.) the weights on MN edges decrease during the training process and are equal to zero in the last training stage, in this study, we focus on attractive forces along NN edges and repulsive forces along FP forces.
An NN edge between two points is \textit{preserved} if these points are placed nearby in the low-dimensional embedding. 
Given the limited capacity of low-dimensional spaces and the complex nearest-neighbor relationships in high-dimensional datasets, it is clear that not all NN edges can be preserved during dimensionality reduction. Further, we do not want to preserve all high-dimensional NNs, 
%since they are Euclidean distance neighbors, not neighbors along the actual manifold of the data. 
since the Euclidean distance neighbors are not always the true neighbors along the actual data manifold. Therefore, the question becomes how to identify which NNs to preserve.

The preservation of specific NN edges depends on intricate dynamics during the embedding optimization, which utilizes the graph extracted from the high-dimensional data. For a pair of NNs $(i,j)$, if they share a greater number of common neighbors, which provide attraction along NN paths, it is more likely that these two points will be positioned close to each other in the low-dimensional space. We hypothesize that these dynamics sometimes correct erroneous edges generated by unreliable high-dimensional distances. If true, this hypothesis implies we can adjust the graph dynamically based on this information and further optimize the embedding using the revised graph.

\begin{figure}[ht]%\begin{figure}[htb]
\begin{center}
\centerline{\includegraphics[width=\columnwidth]
{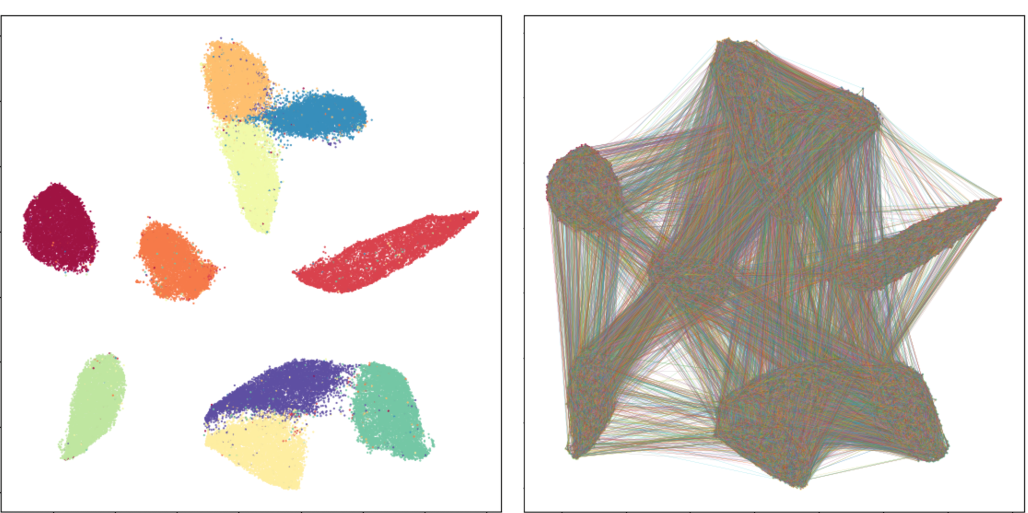}}
\caption{Visualization of NN edge connections of PaCMAP embedding on MNIST dataset.}
\label{fig:MNIST_connect}
\end{center}
\end{figure}

Figure~\ref{fig:MNIST_connect} presents a PaCMAP embedding of the MNIST dataset on the left and the corresponding NN edges' visualization on the right. The figure highlights a substantial number of NN edges bridging distinct clusters that are widely separated. When two points connected by such an NN edge are distant in the final embedding, it suggests that the dynamics of the optimization process have separated them, indicating that the NN edge might be erroneous. Despite this, the erroneous NN edge continues to exert an attractive force, pulling these two points towards each other. If numerous such NN edges exist between two clusters, they can lead to clusters being falsely connected, rather than clearly delineated with distinct boundaries.

\subsection{Insight 2: Insufficient and ineffective negative further point edges fail to create clear boundaries between nearby clusters, particularly in large datasets}
\label{sec:insufficient_FP}

Ideally, the repulsive forces along the FP edges should separate nearby clusters. However, when the dataset becomes larger, with more clusters and larger samples, it is more likely that the sampled FP edges in DR algorithms are insufficient. 
%Again we use PaCMAP as an example.

\begin{theorem}
\label{thm:thm1}
Assume all points between two clusters are approximately equidistant, so that the probability of constructing a positive pair between points from these clusters is constant. The ratio between the number of NN edges to the number of FP edges of PaCMAP between two clusters increases with the number of data points in a dataset.
\end{theorem}
\begin{proof}
Consider a dataset with $n$ data points distributed across $m$ clusters $C_1,C_2,...,C_m$, where each cluster $C_i$ contains $n_i$ data points $(n_1+n_2+...+n_m=n)$. For any two clusters $C_i$ and $C_j$, by assumption, we have:
$$
\forall x_i \in C_i, x_j\in C_j,\quad P(x_i,x_j \text{ are NNs})=p_{ij} 
$$
where $p_{ij}$ is constant. 
%Given that CiC_i and CjC_j are nearby, they are more likely to form NN edges compared to other cluster pairs. 

FP edges for a given point are sampled randomly from all non-NN points. For each point, $n_{FP}$ FP points are randomly selected, where $n_{FP}$ is a constant defaulting to $20$ for PaCMAP. Thus, the expected number of NNs and FPs between $C_i$ and $C_j$ are
\begin{equation}
\begin{aligned}
&\mathbb{E}(\# \text{ of NNs between $C_i$, $C_j$}) = n_i n_j p_{ij}\\
& \mathbb{E}(\# \text{ of FPs between $C_i$, $C_j$}) = \frac{2 n_i n_j n_{\text{FP}}}{n},
\end{aligned}
\end{equation}
because each point $i\in C_i$ selects $\frac{n_j}{n}\cdot n_{\text{FP}}$ FP pairs, the total number of points in $C_i$ sampled from $C_i$ to $C_j$ is $\frac{n_i \cdot n_j \cdot n_{\text{FP}}}{n}$ FP pairs. Similarly, $\frac{n_i \cdot n_j \cdot n_{\text{FP}}}{n}$ FP pairs are sampled from all points in $C_j$ between $C_i$ and $C_j$. Therefore, $\frac{2n_i n_j \cdot n_{\text{FP}}}{n}$ total FP pairs are sampled between these two clusters.
Therefore, the ratio between the number of NN edges and the number of FP edges is
$$
\frac{\mathbb{E}(\# \text{ of NNs between $C_i$, $C_j$})}{\mathbb{E}(\# \text{ of FPs between $C_i$, $C_j$})} = \frac{n \cdot p_{ij}}{2 n_{\text{FP}}}.
$$
Considering that $p_{ij}$ is unaffected by $n$ and $n_{\text{FP}}$, $n_i$ and $n_j$ are constants, the ratio increases with $n$. This result is true for all clusters $C_i$ and $C_j$. Thus, for each pair of clusters, the ratio of NN edges to FP edges grows linearly in $n$. This completes the proof.
\end{proof}

This phenomenon is further illustrated in Figure~\ref{fig:MNIST_part}. When DR is applied to the entire MNIST dataset, images of six different digits are grouped into two large clusters, with each cluster containing three digit classes. However, when DR is applied separately to each of these large clusters, they are successfully separated into distinct clusters, each representing a single digit class. It shows that when the sample size increases, usually more structure is involved in the dataset, which indicates higher complexity and larger number of classes.

\begin{figure}[ht]%\begin{figure}[htb]
\begin{center}
\centerline{\includegraphics[width=\columnwidth]
{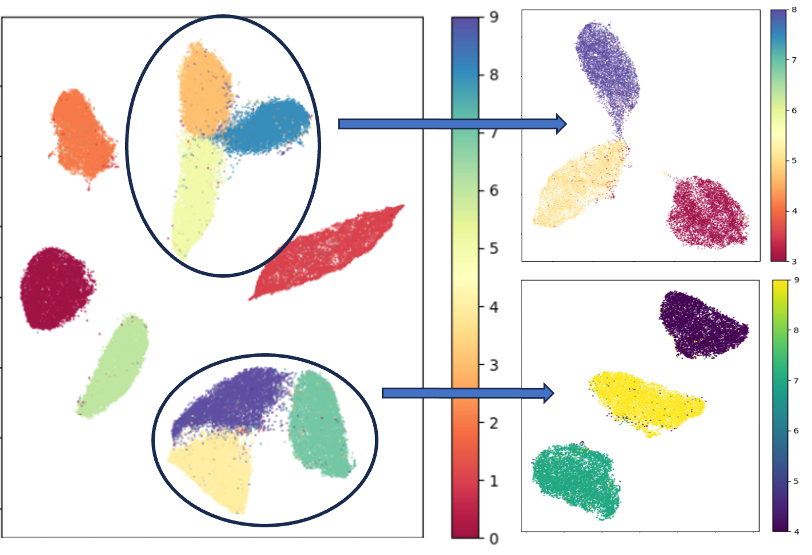}}
\caption{\textit{Left:} PaCMAP embedding on the entire MNIST dataset where six clusters are groups into two large ones (each with three digit classes). \textit{Right}:  PaCMAP embeddings on each of the two groups of three digit classes. The right embeddings work when the left do not because the partial datasets are smaller and thus do not suffer from the problem identified in Insight 2.}
\label{fig:MNIST_part}
\end{center}
\end{figure}

%[TODO (Maybe): The FP repulsive forces fades with low-dim distance which further reduce the effective repulsive forces.]

\section{LocalMAP}
\label{sec:LocalMAP}
The insights above suggest a new approach to DR that keeps information \textit{local}, so as to mitigate the problem of Insight 2, and to decrease the impact of false positive edges, avoiding the pitfall of Insight 1. The way LocalMAP handles this is to \textit{replicate a small-scale DR within its large-scale computation.} Specifically, we increase both NN attractive forces and FP repulsive forces locally compared to standard DR approaches, adjusting weights \textit{dynamically} as we learn more about which data points are false positives. 

This approach has few downsides, if any. Because it sees more local information, and because it reduces the impact of false positive edges, LocalMAP is able to better capture local structure. It is slightly more computationally expensive than some of the  regular DR approaches and faster than others, even for large datasets, 
%LocalMAP is implemented on top of PaCMAP \cite{pacmap}, which is an extremely fast DR approach, so 
%LocalMAP still tends to achieve reasonable run times, 
as we show in Section \ref{sec:runtime}.

%There are other reasons that LocalMAP is built upon PaCMAP -- of the state-of-the-art DR methods, PaCMAP has the advantage in preserving both local and global structure \cite{huang2022towards}, and its loss function has a simple form that is easier to work with than other DR methods. Let us give a brief introduction to PaCMAP before continuing.

\subsection{LocalMAP Algorithm}

As discussed, LocalMAP handles Insight 1 and 2 by increasing local computation. LocalMAP is shown in Algorithm \ref{alg:LocalMAP} with a detailed version in Appendix A. Major differences from previous DR approaches are:

\paragraph{LocalMAP Computation 1.} Adjusting the weights for NN edges dynamically and locally according to the low-dimensional distance during optimization. This is done according to a set of principles listed below. 

\paragraph{LocalMAP Computation 2.} Resampling FP edges to be local. This allows LocalMAP to separate nearby clusters. In practice, local FP edges are randomly selected non-neighbor pairs with distance no larger than the average low-dimension distance among all nearest clusters pairs; we call this the \textit{proximal cluster distance commons} $\bar d_{\text{adj}}$.

Both LocalMAP Computation 1 and LocalMAP Computation 2 apply to the final stage of optimization. LocalMAP uses earlier stages from another DR approach to 
%; PaCMAP's earlier stages 
handle global structure placement, which is sufficient as set up for LocalMAP's computations.

\begin{algorithm*}
    \caption{Implementation of LocalMAP}
    \label{alg:LocalMAP}
    \begin{algorithmic}

    \REQUIRE  $\mathbf{X}$ - data matrix, $\bar d_{\text{adj}}$ from Section \ref{sec:LocalMAP}, parameter $\CNN$$\sim$10. 
    %\qquad\\ 
          %$\mathbf{X}$ - high-dimensional data matrix.
         %$n_{\text{NN}}$, $n_{\text{MN}}$ and $n_{\text{FP}}$ - the number of NN pairs, MN pairs and FP pairs. %$w_{\text{NN}}$, $w_{\text{MN}}$ and $w_{\text{FP}}$: weights for different pairs.
    \ENSURE\quad
        
        $\bullet$ Initialized low-dimensional embedding $\mathbf{Y}$. %Initialize $\mathbf{Y}$ with PCA or random initialization.
        %using PaCMAP's initialization procedure.
        
    \STATE $\bullet$ Construct neighbor pairs, mid-near pairs and further pairs according to high-dimensional distance.

    \STATE $\bullet$ Optimize $\mathbf{Y}$ using the first two training phases of PaCMAP.

    \STATE $\bullet$ Optimize $\mathbf{Y}$ using loss:
%    \begin{enumerate}
%        \item 
%Each NN pair's loss term is multiplied by 
%    weightNN(dij)=a√d2ij+1=a√˜dij,weight_{NN}(d_{ij})=\frac{a}{\sqrt{d_{ij}^2 + 1}} = \frac{a}{\sqrt{\tilde{d}_{ij}}},
%    so the NN loss becomes:   
    %= w_{NN}\cdot\sum_{(i,j) \text{: NN}}\frac{\tilde{d}_{ij}}{10 + \tilde{d}_{ij}} \cdot \frac{a}{\sqrt{\tilde{d}_{ij}}} 
     $$ Loss (\mathbf{Y}) := \sum_{(i,j) \text{: NN}}\frac{\bar d_{\text{adj}} \cdot \sqrt{\tilde{d}_{ij}}}{2(\CNN + \tilde{d}_{ij})} + \sum_{(i,l) \text{: FP}}\frac{1}{\CFP + \tilde{d}_{il}},
    $$
    where  $\tilde{d}_{ij}=\|\by_{i} - \by_{j}\|^2 + 1$.\\
        %\item 
        The FP pairs are resampled every $10$ iterations to stay \textit{local} %throughout the computation, 
        so that all FPs satisfies $\|\by_i-\by_l\|\leq \bar d_{\text{adj}}, \forall (i,l) \in \textit{FP pairs}$.
%    \end{enumerate}
    
    \STATE \textbf{return} $\mathbf{Y}$
    \end{algorithmic}
\end{algorithm*}

\subsection{LocalMAP's principles for NN edge weighting in LocalMap Computation 1}

LocalMAP creates several central aspects for DR. Specifically, these principles are:
\begin{enumerate}
    \item Increased weights for NNs that are close in low-dimensional space. (These are estimated to be true positive pairs.)
    \item Decreased weights for NNs that are far in low-dimensional space. (These are estimated to be false positive pairs.)
    \item Avoid extremely large forces, ensuring stable convergence. 
    \item The weighting function needs to be simple, easy to combine with the NN loss terms, and fast to compute.
\end{enumerate}

\begin{figure}[htb] %\begin{figure}[htb]
\begin{center}
\centerline{\includegraphics[width=.35\textwidth]{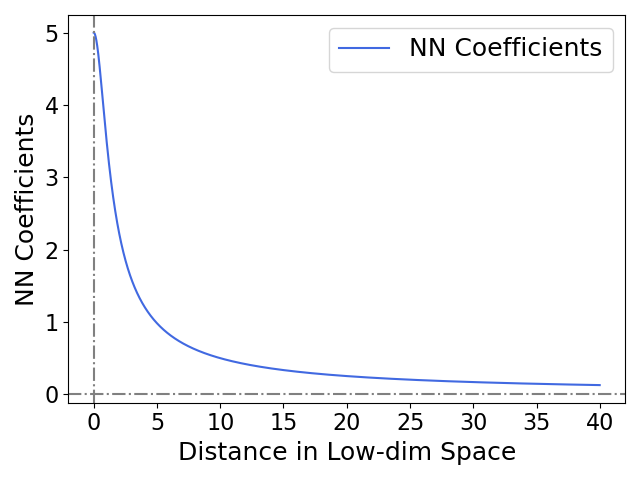}}
\caption{Curve of $\textrm{Coefficient}_{\text{NN}}$. 
%{TODO: remove axis number} 
%which clusters are connected.
}
\label{fig:NN_coef}
\end{center}
\end{figure}

A natural choice to achieve Principles 1 and 2 and 4 is to use a weighting function that is inversely proportional to the low-dimensional distance. However, this could lead to very large values for small low-dimensional distances, making the convergence unstable, violating Principle 3.

Moreover, to define ``close'' in Principles 1 and 2, the embedding scale should be considered. We consider the average distance between adjacent clusters in PaCMAP embedding, $\bar d_{\text{adj}}$, which is approximately 10 based on observations. The midpoint between the two clusters ($\frac{\bar d_{\text{adj}}}{2}$) could serve as a threshold to determine whether to increase or reduce the attractive force. The coefficient should be greater than $1$ when the attraction force needs to be increased and less than $1$ when it needs to be reduced.

To achieve all four principles, we choose a weighting function as follows:
$$
\textrm{Coefficient}_{\text{NN}}(d_{ij}) = \frac{\frac{\bar d_{\text{adj}}}{2}}{\sqrt{d_{ij}^2 + 1}} = \frac{\bar d_{\text{adj}}}{2\sqrt{\tilde{d}_{ij}}}.
$$
% where aa is a constant set to 5, which places this term on the same scale as other terms in PaCMAP's loss function. 
Based on the above equation, when $\frac{\bar d_{\text{adj}}}{2} > \sqrt{\tilde{d}_{ij}} \approx d_{ij}$, the attractive force along the $(i,j)$ pair increases, and this force would decrease when $\frac{\bar d_{\text{adj}}}{2} < \sqrt{\tilde{d}_{ij}} \approx d_{ij}$.

Figure \ref{fig:NN_coef} shows how $\textrm{Coefficient}_{\text{NN}}$ changes with the low dimensional distance between pairs of NN. Implementing this by adapting PaCMAP's loss and setting $\CNN$ to 10 to fix the scale of the embedding, its NN loss term becomes:

\begin{equation}
\begin{aligned}
&Loss_{\text{NN}} = w_{\text{NN}}\cdot\sum_{(i,j) \text{: NN}}\frac{\tilde{d}_{ij}}{\CNN + \tilde{d}_{ij}} \cdot \frac{\bar d_{\text{adj}}}{2\sqrt{\tilde{d}_{ij}}} \\
&= w_{\text{NN}}\cdot\sum_{(i,j) \text{: NN}}\frac{\bar d_{\text{adj}} \cdot \sqrt{\tilde{d}_{ij}}}{2(\CNN + \tilde{d}_{ij})}.
\end{aligned}
\end{equation}

As stated in Section \ref{sec:related_work}, DR methods are unsupervised and no common objective function exists; we proved in Theorem \ref{thm:principles} that LocalMAP also satisfies the six principles mentioned in \citet{pacmap}.

% Since DR methods are unsupervised, no common objective function exists like there would be for supervised learning (e.g., classification error). However, there are principles that reliable DR loss functions typically obey \cite{pacmap}. We will discuss those in Appendix \ref{app:six_principles}, as LocalMAP obeys these principles based on what it inherits from PaCMAP.

\begin{theorem}\label{thm:principles}
LocalMAP's loss function obeys the six principles of \citet{pacmap} for any choices of $\bar d_{\text{adj}}$, excluding NNs that have low dimensional distances larger than a threshold with value $\sqrt{\CNN-1}$ (i.e., possible false positives), where we set $\CNN$$\sim$10 across datasets to determine the scale of the embedding.
\end{theorem}
The proof is given in Appendix C.

\section{Case Study}
Figure \ref{fig:case_study} shows the results for several DR methods on three datasets. Two of the datasets are handwritten digit datasets \cite{lecun2010mnist,USPS}, which means that there are distinct clusters in high dimensions, but also these datasets are special in that Euclidean distance (which is used for defining the graph for DR) is not the same as the geodesic distance along the data manifold, meaning there will be plenty of false positive edges, similar to those shown in Figure \ref{fig:MNIST_connect}. The last dataset is a biological dataset \cite{kang2018multiplexed} that contains multiple cell types that are labeled.

We observe that \textit{LocalMAP does a far better job in separating clusters on all three datasets}. 
A quantitative result is given by the Silhouette score shown in the figures. The definition of Silhouette score is in Appendix E. LocalMAP's scores are superior, confirming what we see visually in these DR plots. There are 5 additional datasets analyzed in Section \ref{sec:experiment} and and visualized in Appendix H, all with similarly impressive visual results. 

\begin{figure}[ht]%\begin{figure}[htb]
\begin{center}
\centerline{\includegraphics[width=\columnwidth]
{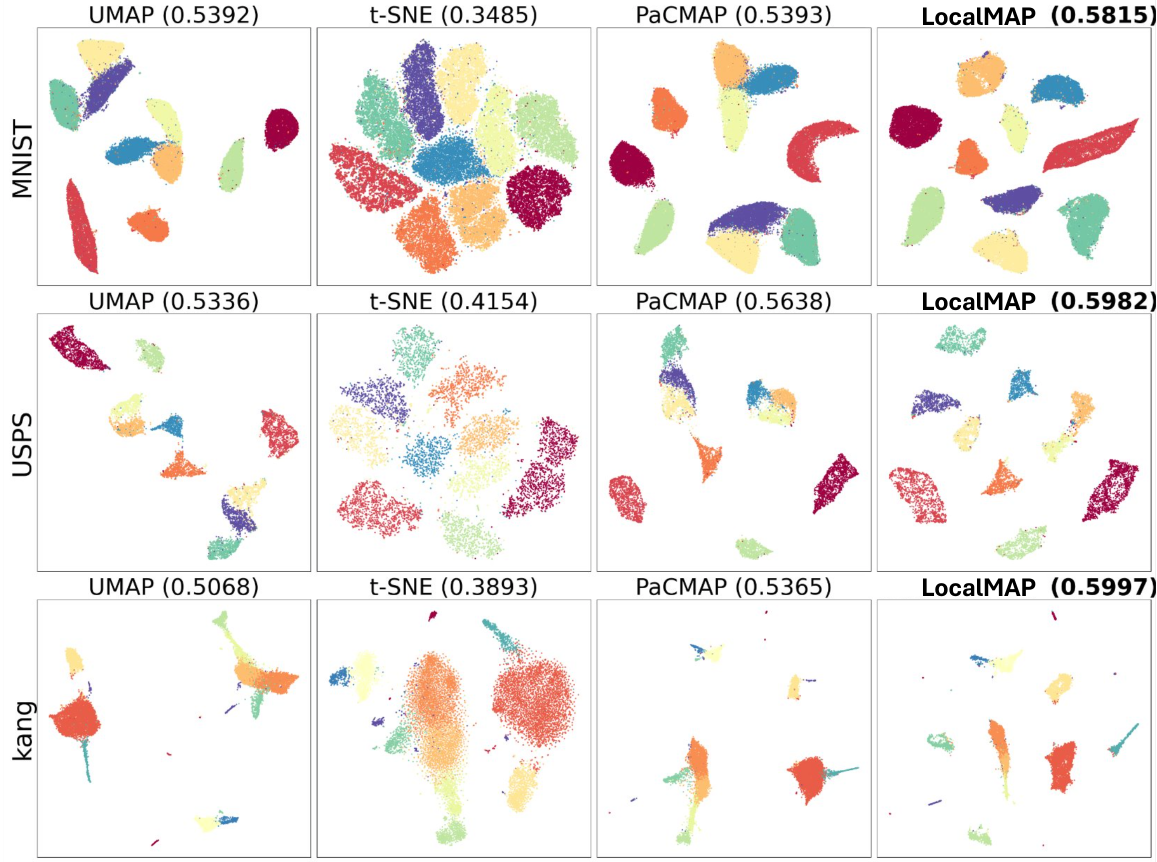}}
\caption{Case study on MNIST \cite{lecun2010mnist}, USPS \cite{USPS} and Kang \cite{kang2018multiplexed}. The Silhouette scores are shown in parentheses.}
\label{fig:case_study}
\end{center}
\end{figure}

\section{Ablation Study}

To understand how each part of the modification contributes to the DR results — an adjustment of NN edge weights and local FP edge resampling — we conducted an ablation study, as shown in Figure~\ref{fig:ablation_MNIST}. The left plot in the figure displays the embedding generated by LocalMAP with adjusted NN edge weights but without local FP resampling. In this embedding, the clusters have higher density, yet the separation between nearby clusters remains inadequate. The right plot illustrates the embedding generated by LocalMAP with local FP resampling but without adjusted NN edge weights. Although the clusters are more dispersed, the separation between nearby clusters is still insufficient. Therefore, these results indicate that both NN edge weight adjustment and local FP edge resampling are necessary to achieve clear separation between clusters. These observations are clear with MNIST, which is why it is useful to work with this dataset; the same observations persist with other datasets.

\begin{figure}[ht]%\begin{figure}[htb]
\begin{center}
\centerline{\includegraphics[width=\columnwidth]
{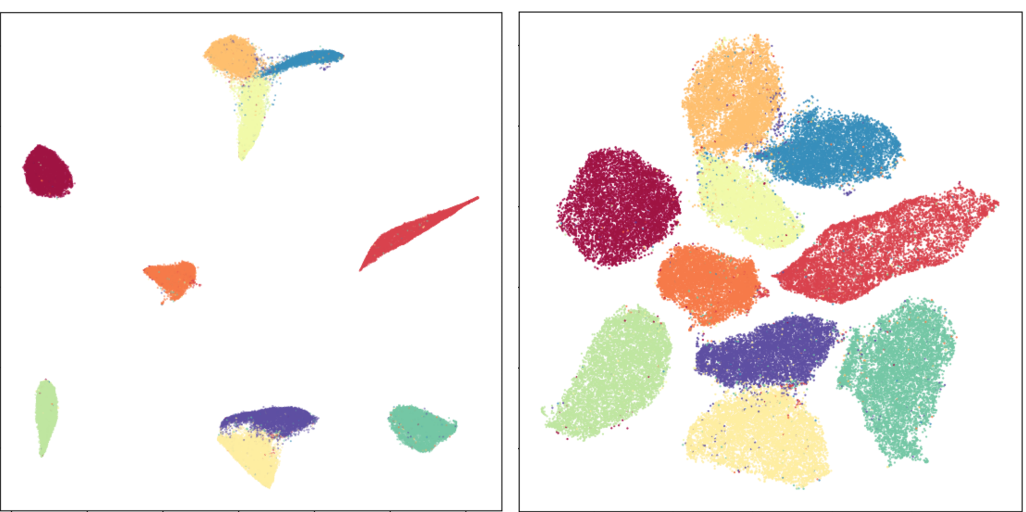}}
\caption{Ablation study of LocalMAP on MNIST dataset. \textbf{Left}: LocalMAP with adjusted NN edge weights but without local FP resampling. \textbf{Right}: LocalMAP with local FP resampling but without adjusted NN edge weights.}
\label{fig:ablation_MNIST}
\end{center}
\end{figure}

\section{Experiments}
\label{sec:experiment}

\subsection{Experimental Setup}

    \subsubsection{Datasets.} Inspired by other DR studies \citep{Tang16, UMAP, TriMAP, huang2022towards}, the following datasets are used to evaluate and compare DR methods: MNIST \citep{lecun2010mnist}, FMNIST \citep{fmnist}, USPS \citep{USPS}, COIL20 \citep{coil20}, 20NG \citep{20NG},  Seurat \citep{stuart2019comprehensive}, Kang \citep{kang2018multiplexed}, Human Cortex \citep{human_cortex_data} and CBMC \citep{neurips2021_data}. \subsubsection{Algorithms.} We evaluate LocalMAP in comparison with several other DR techniques: t-SNE \citep{van_der_Maaten08}, UMAP \citep{UMAP}, PaCMAP \citep{pacmap}, TriMap \citep{TriMAP}, PHATE \citep{phate}, NCVis \cite{artemenkov2020ncvis}, h-NNE \cite{sarfraz2022hnne}, Neg-t-SNE \cite{damrich2023from}, and InfoNC-t-SNE \cite{damrich2023from}. For the t-SNE algorithm, we utilize the openTSNE implementation \cite{openTSNE}. \subsubsection{Computational Environment.} All experiments were run on a 12 Core Intel(R) Xeon(R) CPU E5640 @ 2.67GHz with a GPU RTX2080Ti, with memory limit 32G. All experiments were run 10 times for error bar computation. \subsubsection{Biological Dataset Preprocessing.} For the single-cell datasets, all variables were normalized and log-transformed by the \texttt{SCANPY} package \citep{scanpy}. The top 1000 variance genes within each dataset were selected before applying LocalMAP. For the datasets that have different batches, the ComBat algorithm \citep{combact} was used to remove batch effects. The detailed dataset description are shown within Appendix D.

\subsection{Experimental Results: Silhouette Score}
\label{sec:sil_score}
Evaluation of DR approaches is challenging because DR methods are unsupervised. There are numerous evaluation methods that characterize different aspects of performance; an overview is given by \citet{huang2022towards}, though several DR papers propose their own evaluation metrics \citep[e.g.,][]{TriMAP}. In this work, we use the silhouette score \cite{rousseeuw1987silhouettes} as our evaluation metric because it captures how clearly the clusters are separated, which no other metric captures. The definition of the sihouette score is in Appendix E. 

All the methods are evaluated based on their default hyperparameters. The result in Table \ref{tab:silhouette_score} shows that LocalMAP separates clusters better than other approaches. This table simply quantifies the high quality clusters we see visually in the LocalMAP plots throughout this paper. We have also shown the performance of UMAP, LargeVis, t-SNE, and PHATE with tuned hyperparameters, and additional results are in Table 6 in Appendix J. \textit{LocalMAP shows a better separation among clusters.}

\begin{table*}[ht]
\centering
\caption{Silhouette scores for different Algorithms. \textbf{Bold} is best, \underline{underline} is second best or statistically insignificant from the best. Each row is an algorithm, each column is a dataset.}
\label{tab:silhouette_score}
\renewcommand\arraystretch{1.2}
\scalebox{0.87}{
\begin{tabular}{cccccccccc}
\toprule
 & \textbf{MNIST} & \textbf{FMNIST} & \textbf{USPS} & \textbf{COIL20} & \textbf{20NG} & \textbf{Kang} & \textbf{Seurat} & \begin{tabular}[c]{@{}l@{}}\textbf{Human}\\ \textbf{Cortex} \end{tabular} & \begin{tabular}[c]{@{}l@{}}\textbf{CBMC} \end{tabular} \\
\midrule
\textbf{PCA} & 0.02$\pm$0.00 & -0.03$\pm$0.00 & 0.10$\pm$0.00 & 0.01$\pm$0.00 & -0.19$\pm$0.00 & 0.12$\pm$0.00 & -0.06$\pm$0.00 & -0.08$\pm$0.00 & -0.11$\pm$0.00 \\
\textbf{t-SNE} & 0.35$\pm$0.00 & 0.12$\pm$0.00 & 0.42$\pm$0.00 & 0.41$\pm$0.00 & \underline{-0.11$\pm$0.00} & 0.40$\pm$0.01 & 0.22$\pm$0.01 & 0.11$\pm$0.01 & 0.15$\pm$0.01 \\
\textbf{UMAP} & 0.52$\pm$0.01 & \textbf{0.19$\pm$0.00} & 0.53$\pm$0.00 & \textbf{0.58$\pm$0.01} & -0.15$\pm$0.01 & 0.51$\pm$0.00 & 0.30$\pm$0.00 & 0.12$\pm$0.02 & \underline{0.22$\pm$0.00} \\
\textbf{PaCMAP} & \underline{0.54$\pm$0.01} & \textbf{0.19$\pm$0.00} & \underline{0.56$\pm$0.00} & 0.51$\pm$0.02 & \underline{-0.11$\pm$0.01} & 0.53$\pm$0.00 & \underline{0.31$\pm$0.00} & \underline{0.13$\pm$0.01} & \underline{0.22$\pm$0.00} \\
\textbf{LargeVis} & 0.49$\pm$0.05 & 0.11$\pm$0.03 & 0.41$\pm$0.12 & 0.38$\pm$0.01 & -0.13$\pm$0.01 & 0.44$\pm$0.01 & 0.25$\pm$0.01 & 0.10$\pm$0.02 & 0.17$\pm$0.00 \\
\textbf{TriMAP} & 0.41$\pm$0.00 & \underline{0.17$\pm$0.00} & 0.48$\pm$0.00 & 0.47$\pm$0.00 & -0.13$\pm$0.00 & \underline{0.55$\pm$0.00} & \textbf{0.32$\pm$0.00} & 0.07$\pm$0.00 & 0.21$\pm$0.00 \\
\textbf{PHATE} & 0.26$\pm$0.02 & 0.11$\pm$0.01 & 0.27$\pm$0.01 & 0.33$\pm$0.00 & -0.21$\pm$0.01 & 0.48$\pm$0.02 & 0.27$\pm$0.01 & -0.09$\pm$0.01 & 0.06$\pm$0.01 \\
\textbf{HNNE} & 0.21$\pm$0.03 & 0.06$\pm$0.04 & 0.23$\pm$0.00 & 0.03$\pm$0.00 & -0.34$\pm$0.03 & 0.39$\pm$0.06 & -0.00$\pm$0.03 & -0.09$\pm$0.06 & 0.12$\pm$0.05 \\
\textbf{Neg-t-SNE} & 0.48$\pm$0.00 & \textbf{0.19$\pm$0.00} & 0.48$\pm$0.00 & 0.44$\pm$0.01 & \underline{-0.11$\pm$0.00} & 0.53$\pm$0.00 & \textbf{0.32$\pm$0.00} & 0.12$\pm$0.00 & \textbf{0.24$\pm$0.00} \\
\textbf{NCVis} & 0.38$\pm$0.02 & \textbf{0.19$\pm$0.00} & 0.44$\pm$0.00 & 0.53$\pm$0.00 & -0.15$\pm$0.00 & 0.51$\pm$0.00 & 0.27$\pm$0.00 & 0.10$\pm$0.00 & 0.20$\pm$0.00 \\
\textbf{InfoNC-t-SNE} & 0.33$\pm$0.00 & 0.13$\pm$0.00 & 0.37$\pm$0.00 & 0.43$\pm$0.01 & \underline{-0.11$\pm$0.00} & 0.46$\pm$0.00 & 0.26$\pm$0.00 & 0.10$\pm$0.00 & 0.21$\pm$0.00 \\
\textbf{LocalMAP} & \textbf{0.58$\pm$0.00} & \textbf{0.19$\pm$0.00} & \textbf{0.60$\pm$0.00} & \underline{0.56$\pm$0.01} & \textbf{-0.10$\pm$0.00} & \textbf{0.60$\pm$0.00} & \textbf{0.32$\pm$0.00} & \textbf{0.14$\pm$0.00} & \underline{0.22$\pm$0.00} \\
\bottomrule
\end{tabular}}
\end{table*}

\begin{table*}[ht]
\caption{Running Time for Different Algorithms. Each row is an algorithm, each column is a dataset.}
\label{tab:running_time}
\centering
\renewcommand\arraystretch{1.2}
\scalebox{0.97}{
\begin{tabular}{cccccccccc}

\toprule
 & \textbf{MNIST} & \textbf{FMNIST} & \textbf{USPS} & \textbf{COIL20} & \textbf{20NG} & \textbf{Kang} & \textbf{Seurat} & \begin{tabular}[c]{@{}l@{}}\textbf{Human}\\ \textbf{Cortex} \end{tabular} & \begin{tabular}[c]{@{}l@{}}\textbf{CBMC} \end{tabular} \\
\midrule
\textbf{PCA} & 00:01.77 & 00:01.64 & 00:00.13 & 00:02.77 & 00:00.11 & 00:00.07 & 00:00.07 & 00:00.22 & 00:00.20 \\
\textbf{t-SNE} & 03:02.38 & 02:36.97 & 00:35.92 & 00:10.36 & 01:10.69 & 00:40.33 & 01:05.32 & 01:05.68 & 01:49.60 \\
\textbf{UMAP} & 00:28.05 & 00:33.94 & 00:08.80 & 00:08.76 & 00:08.87 & 00:08.38 & 00:13.51 & 00:27.18 & 00:29.33 \\
\textbf{PaCMAP} & 00:52.39 & 00:34.54 & 00:04.26 & 00:03.28 & 00:08.29 & 00:05.59 & 00:13.16 & 00:18.72 & 00:39.59 \\
\textbf{LargeVis} & 14:51.71 & 14:02.86 & 06:45.63 & 07:09.97 & 06:44.39 & 07:12.39 & 08:15.04 & 08:31.17 & 11:18.19 \\
\textbf{TriMAP} & 00:53.64 & 00:48.16 & 00:06.59 & 00:01.29 & 00:13.05 & 00:13.26 & 00:19.31 & 00:26.75 & 01:03.43 \\
\textbf{PHATE} & 04:15.00 & 01:45.05 & 00:11.40 & 00:05.15 & 00:19.44 & 00:20.12 & 00:42.26 & 01:25.71 & 06:26.80 \\
\textbf{HNNE} & 00:10.05 & 00:06.41 & 00:00.60 & 00:02.35 & 00:02.19 & 00:01.33 & 00:04.21 & 00:03.02 & 00:04.27 \\
\textbf{Neg-t-SNE} & 01:02.65 & 01:10.92 & 00:12.62 & 00:03.86 & 00:23.06 & 00:18.32 & 00:28.25 & 00:42.77 & 00:56.06 \\
\textbf{NCVis} & 02:36.92 & 01:39.70 & 00:12.92 & 00:14.16 & 00:26.69 & 00:31.72 & 00:42.20 & 01:03.09 & 02:16.05 \\
\textbf{InfoNC-t-SNE} & 01:06.19 & 01:01.58 & 00:14.30 & 00:05.63 & 00:25.31 & 00:18.85 & 00:36.02 & 00:46.74 & 01:04.94 \\
\textbf{LocalMAP} & 01:47.35 & 00:47.51 & 00:06.23 & 00:06.77 & 00:13.28 & 00:07.63 & 00:20.17 & 00:29.02 & 01:12.35 \\
\bottomrule
\end{tabular}}
\end{table*}

\subsection{Experimental Results: Posthoc Classification}
DR methods are unsupervised. Here, we consider an evaluation of whether class information, which is not presented to the DR algorithm, is preserved during the process of DR. This is a type of local structure evaluation, but unlike the silhouette score, it does not consider the margins between classes and has several other problems as an evaluation metric, discussed in Appendix F. Essentially, clusters that visually appear merged or are broken up into subclusters (i.e., poor DR plots) can still yield high posthoc classification scores. Table 3 and 7 in Appendix J show the results, which is that LocalMAP achieves similar posthoc classification performance to 11 state-of-the-art DR methods. (But, as discussed, this evaulation measure is problematic.)

\subsection{Experimental Results: Runtime}\label{sec:runtime}
Table \ref{tab:running_time} shows runtimes for several algorithms. LocalMAP is slightly more computationally intensive than other methods because it resamples the graph dynamically, but it is still efficient.
Efficiency is typically not as important as DR quality, as evaluated in Section \ref{sec:sil_score}.

\subsection{Experimental Results: Robustness and Sensitivity}
\label{subsec:sensitivity}

DR results cannot be trusted if they change under random initialization. 
Hence, an important characteristic of DR algorithms is that they produce consistent results with different initial conditions. Figure 17 in Appendix I shows the result of LocalMAP over several runs, showing that it is capable of consistently producing correct clustering results, whereas other methods do not. In fact, \textit{there are no runs of any other methods that distinctly separate the 10 clusters that LocalMAP finds every time.} 

% We show that LocalMAP is not sensitive to small changes in proximity cluster distance commons $\bar d_{\text{adj}}$ in Appendix \ref{app:sensitivity_hyperparameter}.

\section{Discussion}
\label{sec:discussion}
While in the past, DR methods aimed to maintain the local and global structure inherent in high-dimensional data, these methods trusted its underlying graph structure, usually derived from an untrustworthy distance metric. LocalMAP does not do this, instead it dynamically (during run time) discovers what parts of the graph are untrustworthy, and what parts of the graph are not sampled well enough to be able to tell clusters apart. This is why its results are visually and quantitatively better than other DR methods -- it essentially cleans up the data while it is running.

One direction for future work would be to combine the insights of LocalMAP with new parametric approaches to DR that have just begun to yield successful results \cite{parampacmap}.

%In terms of limitations, a challenge common to all DR approaches is that there is no truly effective mathematical theory that underlies them. Many DR methods have underlying probabilistic motivation, which does not translate to success in practice, as we have seen in our experiments and in experiments in many other works. 
%While UMAP \cite{UMAP} has an abundance of mathematics behind its derivation, its loss function can be derived directly using much simpler means. PaCMAP's loss comes from a set of principles, but there are many loss functions that obey the principles (including those of t-SNE and UMAP).
%    An important open problem in the field is to outline principles of what information must be preserved from high dimensions; this knowledge could potentially lead to higher quality DR projections. One important thing to incorporate into such theories (as we have learned from this work) would be how much to trust the graph structure of the data.

Our work could have substantial societal impact, for instance, if it is able to find clusters of patients that have different immune system properties \cite{SemenovaHIV2024,Falcinelli2023}. Our experiments indicate that LocalMAP has a higher chance of accomplishing this than past DR approaches.

\section*{Acknowledgments}
We acknowledge funding from the National Science Foundation under grants DGE-2022040, CMMI-2323978, and IIS-2130250 and the National Institutes of Health (NIH) under grant R01-DA054994 and other transactions award 1OT2OD032701-01.

\bibliography{aaai25}

\appendix
\clearpage
\onecolumn

\section{Detailed LocalMAP Algorithm}
\label{app:algo_detail}

\begin{algorithm*}
    \caption*{Algorithm 1: Implementation of LocalMAP}
    \begin{algorithmic}

    \REQUIRE  $\mathbf{X}$ - data matrix, $\bar d_{\text{adj}}$ from Section \ref{sec:LocalMAP}, parameter $\CNN$$\sim$10. 
    %\qquad\\ 
          %$\mathbf{X}$ - high-dimensional data matrix.
         %$n_{\text{NN}}$, $n_{\text{MN}}$ and $n_{\text{FP}}$ - the number of NN pairs, MN pairs and FP pairs. %$w_{\text{NN}}$, $w_{\text{MN}}$ and $w_{\text{FP}}$: weights for different pairs.
    \ENSURE\quad
        
        $\bullet$ $\mathbf{Y}$ - low-dimensional data matrix. Initialize $\mathbf{Y}$ with PCA or random initialization.
        %using PaCMAP's initialization procedure.
        
    \STATE $\bullet$ Construct neighbor pairs, mid-near pairs and further pairs according to high-dimensional distance.

    \STATE $\bullet$ Optimize the low-dimensional embedding $\mathbf{Y}$ using the first two training phases of PaCMAP.

    \STATE $\bullet$ For the last training phase, optimize the low-dimensional embedding $\mathbf{Y}$ using the loss:
%    \begin{enumerate}
%        \item 
%Each NN pair's loss term is multiplied by 
%    weightNN(dij)=a√d2ij+1=a√˜dij,weight_{NN}(d_{ij})=\frac{a}{\sqrt{d_{ij}^2 + 1}} = \frac{a}{\sqrt{\tilde{d}_{ij}}},
%    so the NN loss becomes:   
    %= w_{NN}\cdot\sum_{(i,j) \text{: NN}}\frac{\tilde{d}_{ij}}{10 + \tilde{d}_{ij}} \cdot \frac{a}{\sqrt{\tilde{d}_{ij}}} 
     $$ Loss (\mathbf{Y}) := \sum_{(i,j) \text{: NN}}\frac{\bar d_{\text{adj}} \cdot \sqrt{\tilde{d}_{ij}}}{2(\CNN + \tilde{d}_{ij})} + \sum_{(i,l) \text{: FP}}\frac{1}{\CFP + \tilde{d}_{il}},
    $$
    where  $\tilde{d}_{ij}=\|\by_{i} - \by_{j}\|^2 + 1$.
        %\item 
        We use the Adam optimizer, and the FP pairs are resampled every $10$ iterations to stay \textit{local} throughout the computation, so that all FPs have low-dimensional distance within $\bar d_{\text{adj}}$ if possible, i.e., $\|\by_i-\by_l\|\leq \bar d_{\text{adj}}$ for all $(i,l)$ FP pairs. (We impose a limit of 20 maximum sampling attempts for the FPs due to computational considerations.) 
%    \end{enumerate}
    
    \STATE \textbf{return} $\mathbf{Y}$
    \end{algorithmic}
\end{algorithm*}

\clearpage
\section{Detailed proof of Theorem \ref{thm:thm1}}
\begin{proof}
Considering a dataset with $n$ data points distributed across $m$ clusters $C_1,C_2,...,C_m$, where each cluster $C_i$ contains $n_i$ data points $(n_1+n_2+...+n_m=n)$. For any two clusters $C_i$ and $C_j$, by assumption, we have:
$$
\forall x_i \in C_i, x_j\in C_j,\quad P(x_i,x_j \text{ are NNs})=p_{ij} 
$$
where $p_{ij}$ is constant. 
%Given that CiC_i and CjC_j are nearby, they are more likely to form NN edges compared to other cluster pairs. 

FP edges for a given point are sampled randomly from all non-NN points. For each point, $n_{FP}$ FP points are randomly selected, where $n_{FP}$ is a constant defaulting to $20$ for PaCMAP. Thus, the expected number of NNs and FPs between $C_i$ and $C_j$ are
\begin{equation*}
\begin{aligned}
&\mathbb{E}(\# \text{ of NNs between $C_i$, $C_j$}) = \sum_{x_i \in C_i, x_j\in C_j} P(x_i,x_j \text{ are NNs}) = \sum_{x_i \in C_i, x_j\in C_j} p_{ij} = n_i n_j p_{ij}\\
& \mathbb{E}(\# \text{ of FPs between $C_i$, $C_j$}) = \sum_{x_i \in C_i} \frac{n_j}{n} \cdot n_{FP} + \sum_{x_j \in C_j} \frac{n_i}{n} \cdot n_{FP} = \frac{2 n_i n_j n_{\text{FP}}}{n},
\end{aligned}
\end{equation*}
because each point $i\in C_i$ selects $\frac{n_j}{n}\cdot n_{\text{FP}}$ FP pairs, the total number of points in $C_i$ sampled from $C_i$ to $C_j$ is $\frac{n_i \cdot n_j \cdot n_{\text{FP}}}{n}$ FP pairs. Similarly, $\frac{n_i \cdot n_j \cdot n_{\text{FP}}}{n}$ FP pairs are sampled from all points in $C_j$ between $C_i$ and $C_j$. Therefore, $\frac{2n_i n_j \cdot n_{\text{FP}}}{n}$ total FP pairs are sampled between these two clusters.
Therefore, the ratio between the number of NN edges and the number of FP edges is
$$
\frac{\mathbb{E}(\# \text{ of NNs between $C_i$, $C_j$})}{\mathbb{E}(\# \text{ of FPs between $C_i$, $C_j$})} = \frac{n_i n_j p_{ij}}{2 n_i n_j n_{FP}/n} = \frac{n \cdot p_{ij}}{2 n_{\text{FP}}}.
$$
Considering that $p_{ij}$ is unaffected by $n$ and $n_{\text{FP}}$, $n_i$ and $n_j$ are constants, the ratio increases with $n$. This result is true for all clusters $C_i$ and $C_j$. Thus, for each pair of clusters, the ratio of NN edges to FP edges grows linearly in $n$. This completes the proof.
\end{proof}

\clearpage
\section{Proof of Theorem \ref{thm:principles}: six principles for DR loss functions}
\label{app:six_principles}
Here we check the conditions identified by \cite{pacmap} for high-quality DR loss functions. Consider a triplet $(i,j,k)$ where $i$ and $j$ are high-dimensional neighbors that should be attracted to each other, and $i$ and $k$ are further points in the high-dimension that should be repulsed from each other. How forces along pairs $(i,j)$ and $(i,k)$ should be affected by the distance between these pairs of points $d_{ij}$ and $d_{ik}$ are defined by the six principles.

LocalMAP's loss in the early stages follows these principles automatically. For its last stage, the loss function follows the principles for NNs that are close in low-dimensional space (which indicates they are true positive pairs) with distance smaller than $\sqrt{\CNN-1} = 3$. For NNs beyond that, we do not want them to obey the principles because they are probably false positives.

According to Proposition $1$ of \cite{pacmap}, for loss functions of the form:
\[
Loss = 
\sum_{ij} Loss_{\textrm{attractive}}(d_{ij}) + 
\sum_{ik} Loss_{\textrm{repulsive}}(d_{ik}),
\]
where its derivatives are:
\[
f(d_{ij}):=\frac{\partial Loss_{\textrm{attractive}}(d_{ij})}{\partial d_{ij}}, \quad
g(d_{ik}):=-\frac{\partial Loss_{\textrm{repulsive}}(d_{ik})}{\partial d_{ik}},
\]
each of the six principles is obeyed under the following conditions of the loss function's derivatives:

\begin{enumerate}
    \item The functions $f(d_{ij})$ and $g(d_{ik})$ are non-negative and unimodal.%\Yaron{left-skewed is not needed for the proof}
    \item $\lim_{d_{ij}\rightarrow 0}f(d_{ij})=\lim_{d_{ij}\rightarrow \infty}f(d_{ij})=0$, 
    $\lim_{d_{ik}\rightarrow 0}g(d_{ik})=\lim_{d_{ik}\rightarrow \infty}g(d_{ik})=0$.  
\end{enumerate}

For LocalMAP, the loss function for earlier stages is adopted from PaCMAP and thus obeys the principles. For LocalMAP's last phase, the loss function is:
\[
Loss_{\textrm{attractive}}(d_{ij})=\frac{\bar d_{\text{adj}} \cdot \sqrt{\tilde{d}_{ij}}}{2(\CNN + \tilde{d}_{ij})}, \quad \tilde{d}_{ij} =d_{ij}^2 + 1 
\]
and thus
\[
f(d_{ij}) = \frac{\partial Loss_{\textrm{attractive}}(d_{ij})}{\partial d_{ij}} = \frac{\bar d_{\text{adj}} \cdot d_{ij}(\CNN - d_{ij}^2 - 1)}{2(d_{ij}^2 + 1)^{\frac{1}{2}}(\CNN + d_{ij}^2 + 1)} \ge 0 \quad\text{   when   } 0\le d_{ij} \le \sqrt{\CNN-1}. 
\]
To show $f(d_{ij})$ is unimodal:
$$
\frac{\partial f(d_{ij})}{\partial d_{ij}} = \frac{\bar d_{\text{adj}} (d_{ij}^2+1)^{-\frac{1}{2}} (\CNN + d_{ij}^2 + 1) (2 d_{ij}^6 + 3 d_{ij}^4 - 6 \CNN d_{ij}^4 - 6\CNN d_{ij}^2 + \CNN^2 - 1)}{2(d_{ij}^2 + 1)(\CNN + d_{ij}^2 + 1)^4}.
$$
Since $\bar d_{\text{adj}} (d_{ij}^2+1)^{-\frac{1}{2}} (\CNN + d_{ij}^2 + 1)$ and $(d_{ij}^2 + 1)(\CNN + d_{ij}^2 + 1)^4$ are always positive given $\bar d_{\text{adj}}$ is positive, the sign of $\frac{\partial f(d_{ij})}{\partial d_{ij}}$ depends on $\Delta = 2 d_{ij}^6 + 3 d_{ij}^4 - 6 \CNN d_{ij}^4 - 6\CNN d_{ij}^2 + \CNN^2 - 1$. 

Considering $\Delta$ is a six-degree polynomial of $d_{ij}$, $\Delta=0$ has $6$ roots in the complex field. Let $t=d_{ij}^2$, then 
$$\Delta = 2t^3 + 3t^2 - 6 \CNN t^2 - 6 \CNN t + \CNN^2 - 1=0.
$$ 
It is clear that
$\lim_{t\rightarrow -\infty} \Delta= -\infty, \lim_{t\rightarrow \infty} \Delta=\infty$. When $t=0$, $\Delta = \CNN-1 > 0$, when $t=\CNN - 1$, $\Delta < 0$. Then $t$ has a root in each of $(-\infty,0), (0,\CNN-1)$ and $(\CNN-1, \infty)$, where each root of $t$ corresponds to two roots of $d_{ij}$.

For the one negative root of $t$, it corresponds to two complex root of $d_{ij}$. For the two roots of $t$ in $(0, \CNN-1)$ and $(\CNN-1, \infty)$, each corresponds to one positive root of $d_{ij}$ and one negative root of $d_{ij}$. Therefore, $d_{ij}$ has only one root of $\Delta$ in $(0, \sqrt{\CNN-1})$, and $\frac{\partial f(d_{ij})}{\partial d_{ij}}$ is first positive and then negative in $(0, \sqrt{\CNN-1})$, which indicates that $d_{ij}$ first increases and then decreases in $(0, \sqrt{\CNN-1})$ and thus is unimodal in $(0, \sqrt{\CNN-1})$. 

LocalMAP's loss for FP pairs has an analogous proof, so we are done. \qed

\clearpage
\section{Data Description}\label{sec:descprition}

Detailed descriptions of the data set are shown as follows:

\begin{table}[ht]
\centering
\begin{tabular}{ccc}
\hline
\textbf{Dataset} & \textbf{\# of samples} & \textbf{\# of dimensions} \\ \hline
MNIST & 70,000 & 784 \\
FMNIST & 70,000 & 784 \\
USPS & 9,298 & 256 \\
COIL20 & 1,440 & 16384 \\
20NG & 18,846 & 100 \\
Kang & 13,999 & 1000 \\
Seurat & 30,672 & 1000 \\
Human Cortex & 43,349 & 1000 \\
CBMC & 67,686 & 1000 \\ \hline
\end{tabular}
\end{table}

\section{Silhouette Score}
\label{app:silhouette}

The Silhouette score was originally used to evaluate the quality of clusters in clustering analysis where the ground truth labels are not present.
In this paper, we use the Silhouette score to evaluate cluster quality from an unsupervised algorithm using ground truth labels from supervised data by substituting the clusters generated from clustering algorithms to ground truth class labels. In this case, it measures the embedding's within-class cohesion (measured by a point's average distance to other points in the same class) against between-class separation (measured by a point's minimum average distance to points of different classes). It is calculated as follows:

\begin{enumerate}
    \item For each data point $i$ in a dataset:
 Calculate average distance $a_i$ between \(i\) and all other data points within the same class $C_i$: $a_i = \sum_{j \in C_i, j \neq i} d(i, j) / (|C_i| - 1)$
%      \[a_i = \frac{\sum_{j \in C_i, j \neq i} d(i, j)}{|C_i| - 1} \]
      where \(d(i, j)\) is the distance between data points \(i\) and \(j\).

Then calculate the minimum average distance 
      $b_i$ of $i$ to all data points in a different class $C_j$:  $b_i = \min_{k \neq i} \sum_{k \in C_k} d(i, k)/|C_k|$.
 %     \[b_i = \min_{k \neq i} \frac{\sum_{k \in C_k} d(i, k)}{|C_k|}\]
    \item Calculate the silhouette score $S_i$ for data point $i$:
  $S_i = \frac{b_i - a_i}{\max(a_i, b_i)}$.
    \item The overall silhouette score $S$ for the entire dataset is the average of the silhouette scores for all data points:
$S = \frac{\sum_{i} S_i}{N}$
  where $N$ is the number of data points in the dataset.
\end{enumerate}

\clearpage
\section{Posthoc Classification}\label{sec:posthoc}
Tables \ref{tab:svm_score} show posthoc classification results using SVM. Here, SVM is optimized globally. Posthoc classification aims to determine whether class labels are maintained during DR projection even though labels are not used during the process of DR. 
LocalMAP's performance is comparable with other local DR methods with respect to these scores. 

There are several problems with using posthoc classification as a performance metric. One problem is that its results are not consistent between classification methods, as we can see from these tables. It also does not handle global structure at all \citep{huang2022towards}, meaning that one true cluster could be broken into many subclusters by DR (which would be a poor DR result) and the posthoc classification score could be unchanged. Posthoc classification does not measure margins between clusters. Thus, two true clusters that appear stuck together visually (as a fault of DR) may receive the same posthoc classification score as if they were well separated.

\begin{table}[ht]
\caption{SVM Score for Different Algorithms, \textbf{Bold} is best, \underline{underline} is not significantly different from best (with only 1\% difference). Each row is an algorithm, each column is a dataset.}
\label{tab:svm_score}
\centering
\scalebox{0.9}{
\begin{tabular}{cccccccccc}
\toprule
 & \textbf{MNIST} & \textbf{FMNIST} & \textbf{USPS} & \textbf{COIL20} & \textbf{20NG} & \textbf{Kang} & \textbf{Seurat} & \begin{tabular}[c]{@{}l@{}}\textbf{Human}\\ \textbf{Cortex} \end{tabular} & \begin{tabular}[c]{@{}l@{}}\textbf{CBMC} \end{tabular} \\
\midrule
\textbf{PCA} & 0.47$\pm$0.00 & 0.55$\pm$0.00 & 0.56$\pm$0.00 & 0.66$\pm$0.00 & 0.15$\pm$0.00 & 0.73$\pm$0.00 & 0.46$\pm$0.00 & 0.57$\pm$0.00 & 0.44$\pm$0.00 \\
\textbf{t-SNE} & \textbf{0.97$\pm$0.00} & \underline{0.74$\pm$0.00} & \textbf{0.96$\pm$0.00} & \textbf{0.85$\pm$0.01} & 0.45$\pm$0.01 & \underline{0.95$\pm$0.00} & \underline{0.84$\pm$0.00} & \textbf{0.82$\pm$0.00} & 0.82$\pm$0.00 \\
\textbf{UMAP} & \textbf{0.97$\pm$0.00} & \underline{0.74$\pm$0.01} & \underline{0.95$\pm$0.00} & 0.82$\pm$0.01 & 0.44$\pm$0.01 & \underline{0.95$\pm$0.00} & 0.83$\pm$0.00 & \underline{0.81$\pm$0.00} & \underline{0.82$\pm$0.00} \\
\textbf{PaCMAP} & \textbf{0.97$\pm$0.00} & \underline{0.74$\pm$0.00} & \underline{0.95$\pm$0.00} & 0.83$\pm$0.01 & \underline{0.46$\pm$0.01} & \underline{0.95$\pm$0.00} & \textbf{0.85$\pm$0.00} & \underline{0.81$\pm$0.00} & \textbf{0.83$\pm$0.00} \\
\textbf{LargeVis} & \underline{0.96$\pm$0.00} & \underline{0.74$\pm$0.01} & 0.92$\pm$0.06 & 0.80$\pm$0.02 & \textbf{0.47$\pm$0.00} & \underline{0.95$\pm$0.00} & \underline{0.84$\pm$0.00} & \textbf{0.82$\pm$0.00} &\underline{0.82$\pm$0.00} \\
\textbf{TriMAP} & \underline{0.96$\pm$0.00} & 0.73$\pm$0.00 & \underline{0.95$\pm$0.00} & 0.77$\pm$0.01 & 0.42$\pm$0.01 & \underline{0.95$\pm$0.00} & \underline{0.84$\pm$0.00} & 0.79$\pm$0.00 & \underline{0.82$\pm$0.00} \\
\textbf{PHATE} & 0.86$\pm$0.02 & 0.66$\pm$0.01 & 0.86$\pm$0.01 & \textbf{0.84$\pm$0.00} & 0.33$\pm$0.01 & 0.92$\pm$0.00 & 0.77$\pm$0.00 & 0.70$\pm$0.01 & 0.72$\pm$0.01 \\
\textbf{HNNE} & 0.84$\pm$0.03 & 0.68$\pm$0.01 & 0.82$\pm$0.00 & 0.63$\pm$0.00 & 0.24$\pm$0.05 & 0.90$\pm$0.01 & 0.74$\pm$0.01 & 0.68$\pm$0.03 & 0.73$\pm$0.04 \\
\textbf{Neg-t-SNE} & \underline{0.96$\pm$0.00} & \underline{0.74$\pm$0.00} & 0.93$\pm$0.00 & 0.81$\pm$0.01 & 0.43$\pm$0.01 & \underline{0.95$\pm$0.00} & \underline{0.84$\pm$0.00} & \underline{0.81$\pm$0.00} & \underline{0.82$\pm$0.00} \\
\textbf{NCVis} & 0.94$\pm$0.01 & 0.73$\pm$0.00 & 0.92$\pm$0.00 & 0.79$\pm$0.00 & 0.36$\pm$0.01 & 0.94$\pm$0.00 & 0.83$\pm$0.00 & \textbf{0.82$\pm$0.00} & \underline{0.82$\pm$0.00} \\
\textbf{InfoNC-t-SNE} & \underline{0.96$\pm$0.00} & \underline{0.74$\pm$0.00} & 0.93$\pm$0.00 & 0.82$\pm$0.01 & 0.42$\pm$0.00 & \underline{0.95$\pm$0.00} & \textbf{0.85$\pm$0.00} &\underline{0.81$\pm$0.00} & \textbf{0.83$\pm$0.00} \\
\textbf{LocalMAP} & \textbf{0.97$\pm$0.00} & \textbf{0.75$\pm$0.00} & \textbf{0.96$\pm$0.00} & \underline{0.83$\pm$0.01} & \underline{0.46$\pm$0.01} & \textbf{0.96$\pm$0.00} & \underline{0.84$\pm$0.00} & \underline{0.81$\pm$0.00} & \underline{0.82$\pm$0.00} \\
\bottomrule
\end{tabular}}
\end{table}

\clearpage
\section{The relationship between run time and the number of samples}

Figure \ref{fig:runtime_size} shows the relationship between the number of samples and the run time. 
\begin{figure}[ht]
    \centering
    \includegraphics[width=0.85\linewidth]{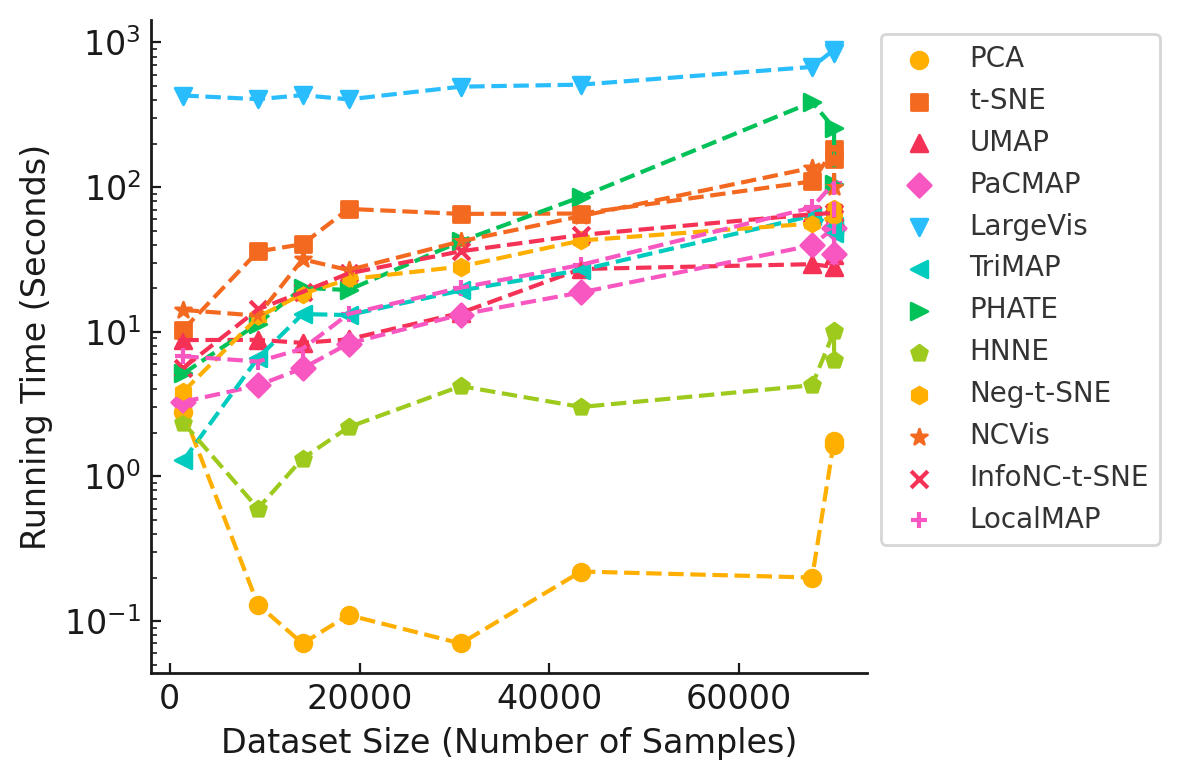}
    \caption{The relationship between the number of samples and the run time (log-scaled in seconds).}
    \label{fig:runtime_size}
\end{figure}

\clearpage
\section{Additional Visualization for Datasets}
\label{app:additional_datasets_viz}

We show how different DR methods perform on different datasets in Figures \ref{fig:mnist_embed} to \ref{fig:neurips_embed}.

\begin{figure}[ht]
    \centering
    \includegraphics[width=\textwidth]{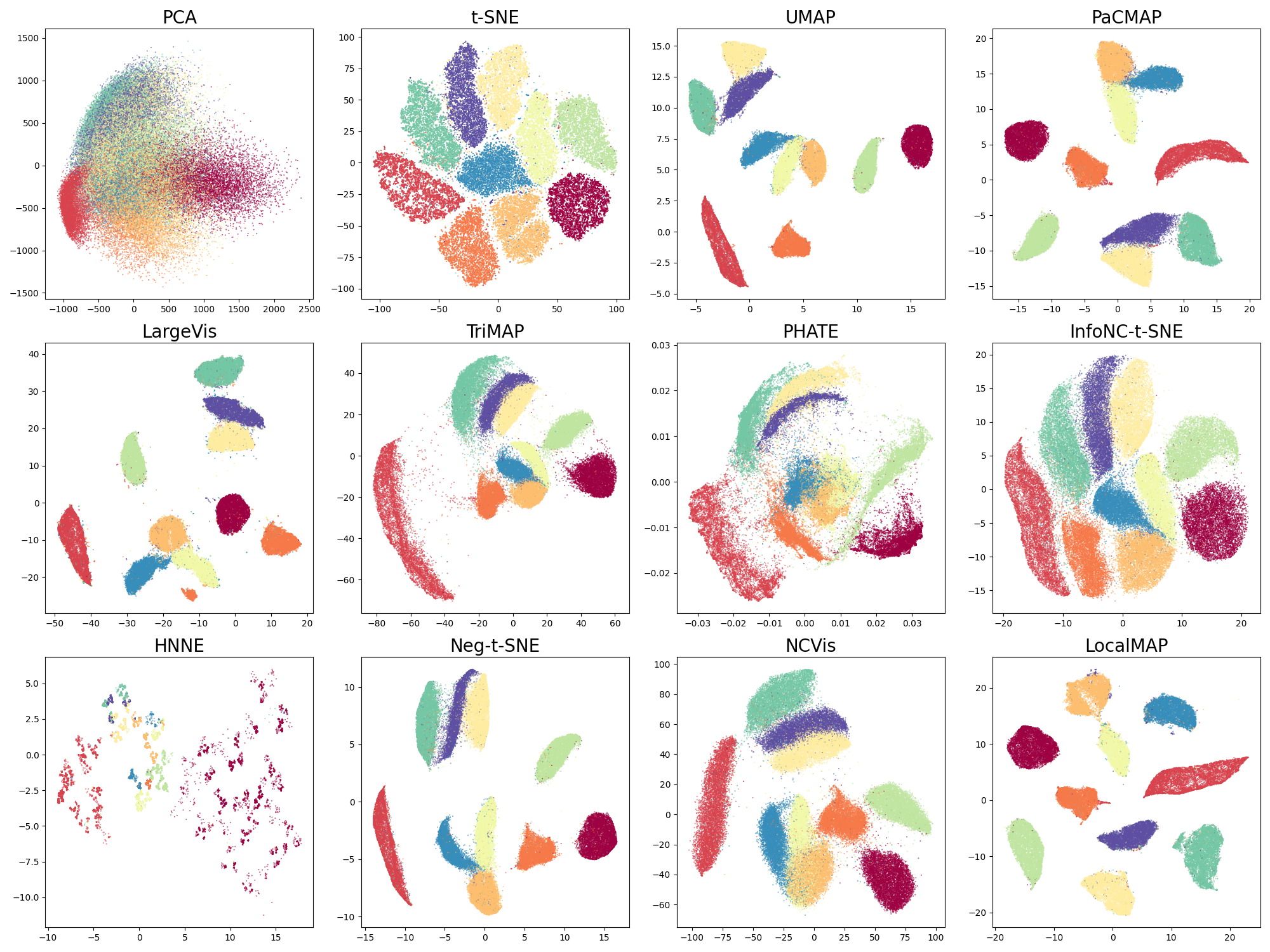}
    \caption{Different DR embeddings for MNIST \citep{lecun2010mnist}}.
    \label{fig:mnist_embed}
\end{figure}

\begin{figure*}[ht]
    \centering
    \includegraphics[width=\textwidth]{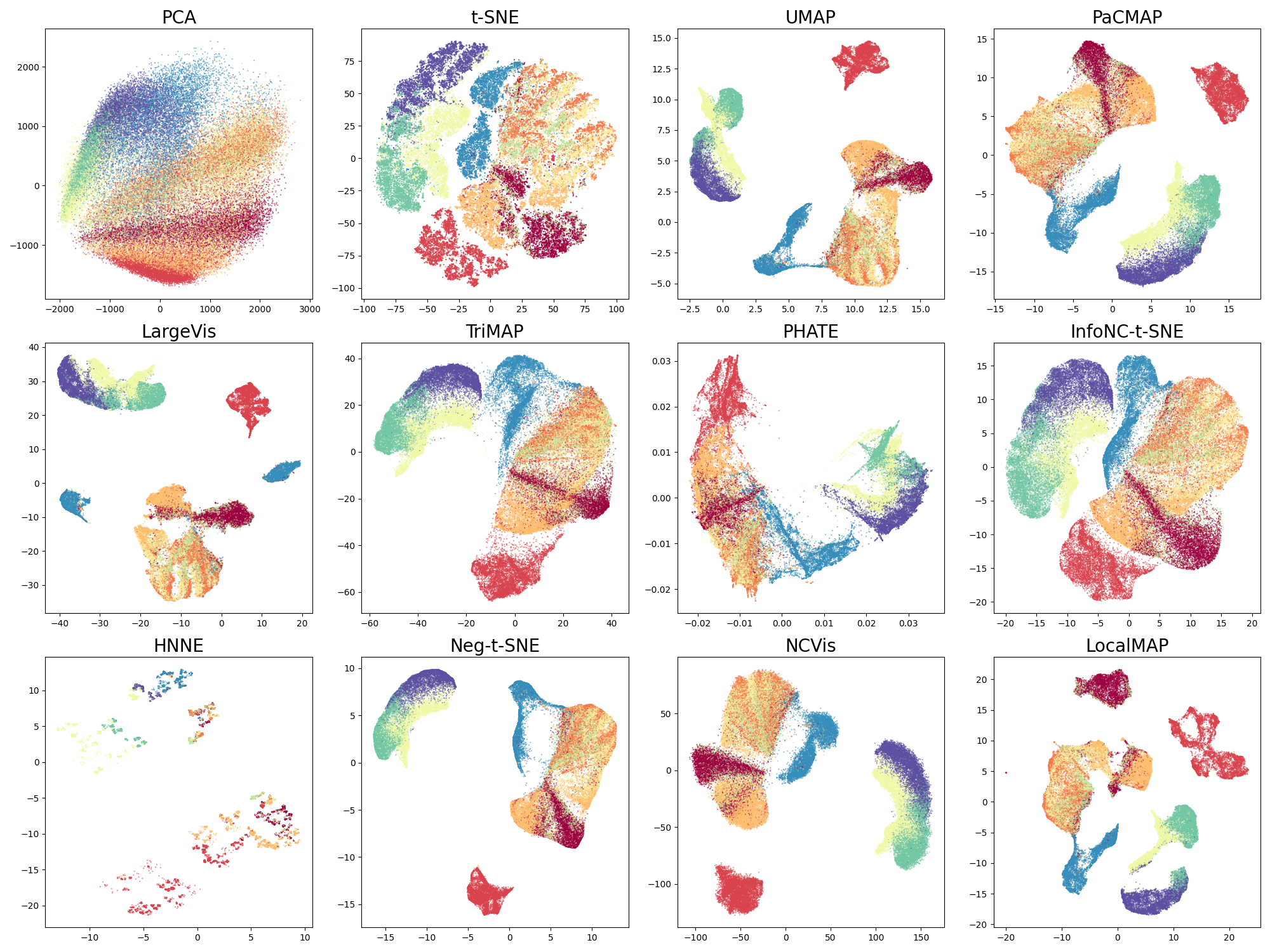}
    \caption{Different DR embeddings for FMNIST \citep{fmnist}}
    \label{fig:fmnist_embed}
\end{figure*}

\begin{figure*}[ht]
    \centering
    \includegraphics[width=\textwidth]{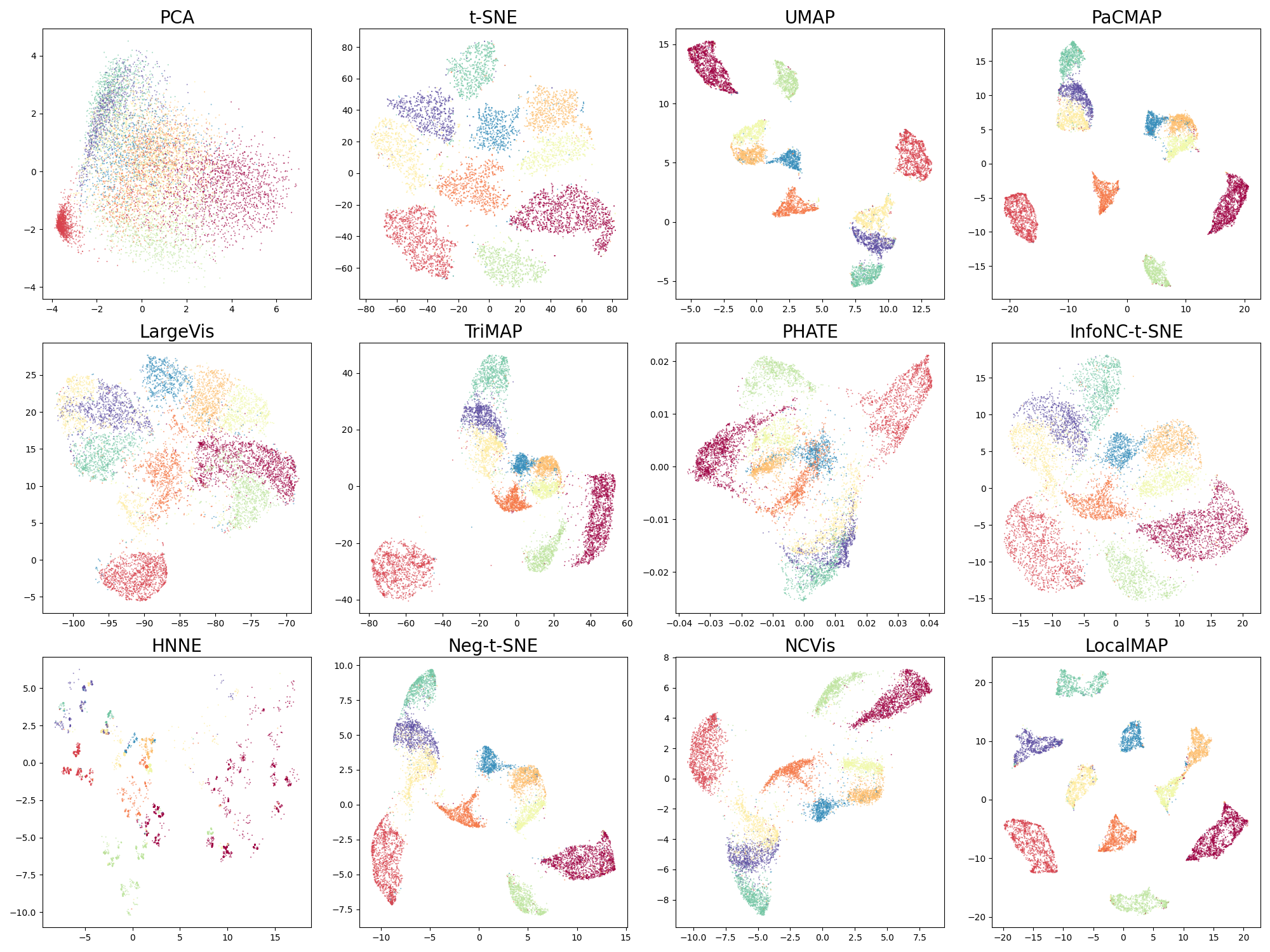}
    \caption{Different DR embeddings for USPS \citep{USPS}}
    \label{fig:usps_embed}
\end{figure*}

\begin{figure*}[ht]
    \centering
    \includegraphics[width=\textwidth]{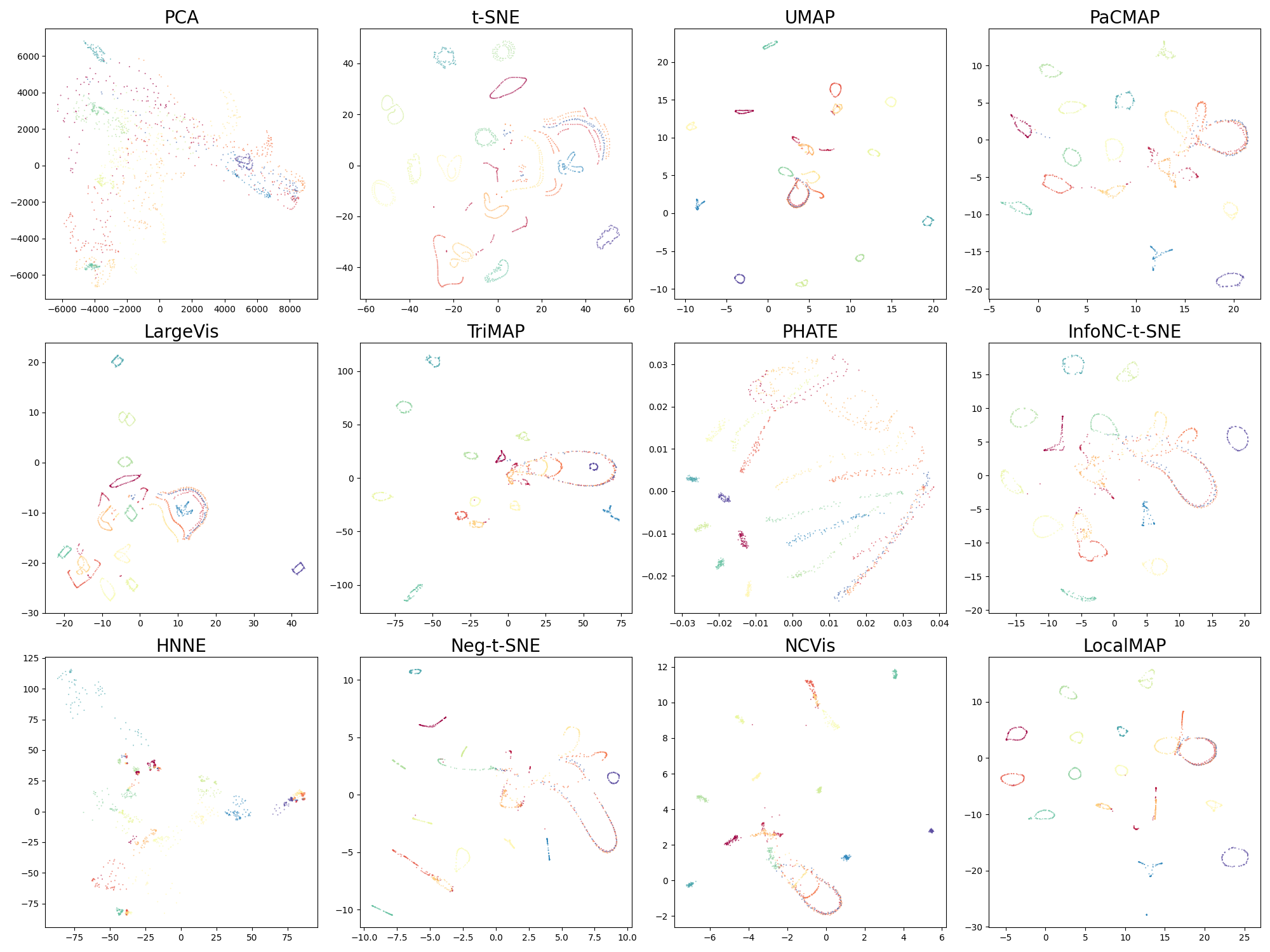}
    \caption{Different DR embeddings for COIL20 \citep{coil20}}
    \label{fig:coli20_embed}
\end{figure*}

\begin{figure*}[ht]
    \centering
    \includegraphics[width=\textwidth]{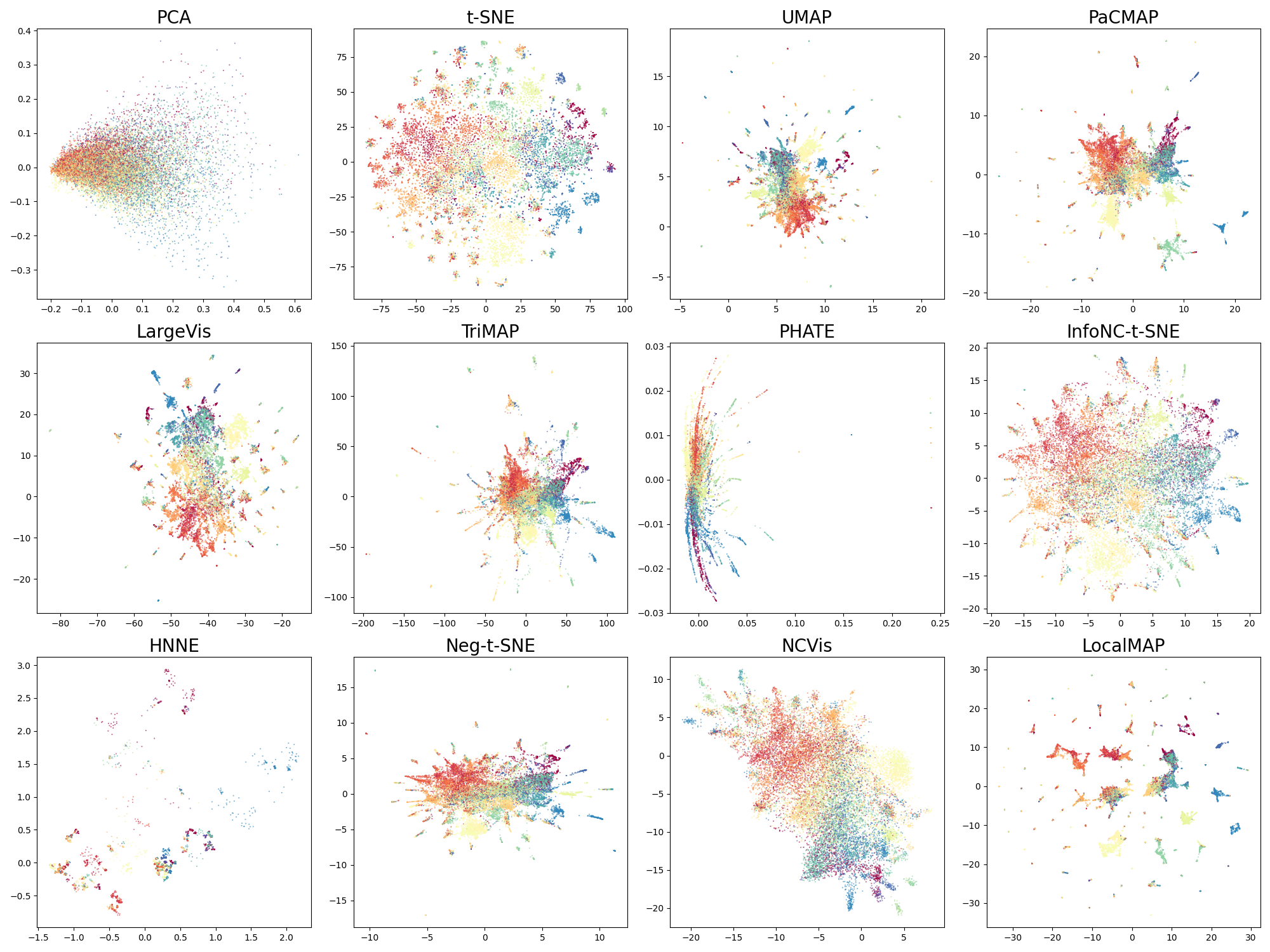}
    \caption{Different DR embeddings for 20NG \citep{20NG}. Newsgroups tend to be intertwined.}
    \label{fig:20ng_embed}
\end{figure*}

\begin{figure*}[ht]
    \centering
    \includegraphics[width=\textwidth]{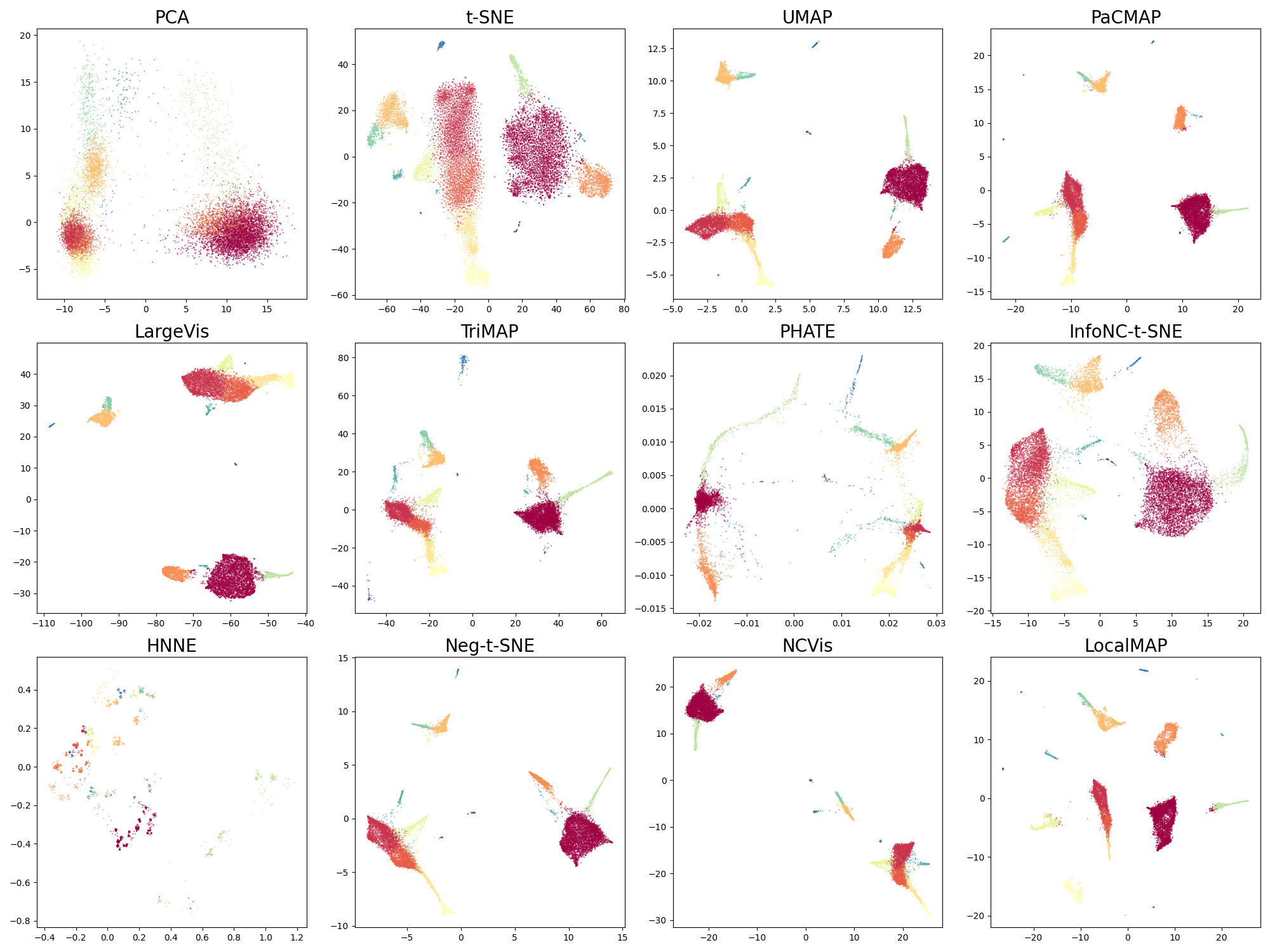}
    \caption{Different DR embeddings for Kang \citep{kang2018multiplexed}}
    \label{fig:kang_embed}
\end{figure*}

\begin{figure*}[ht]
    \centering
    \includegraphics[width=\textwidth]{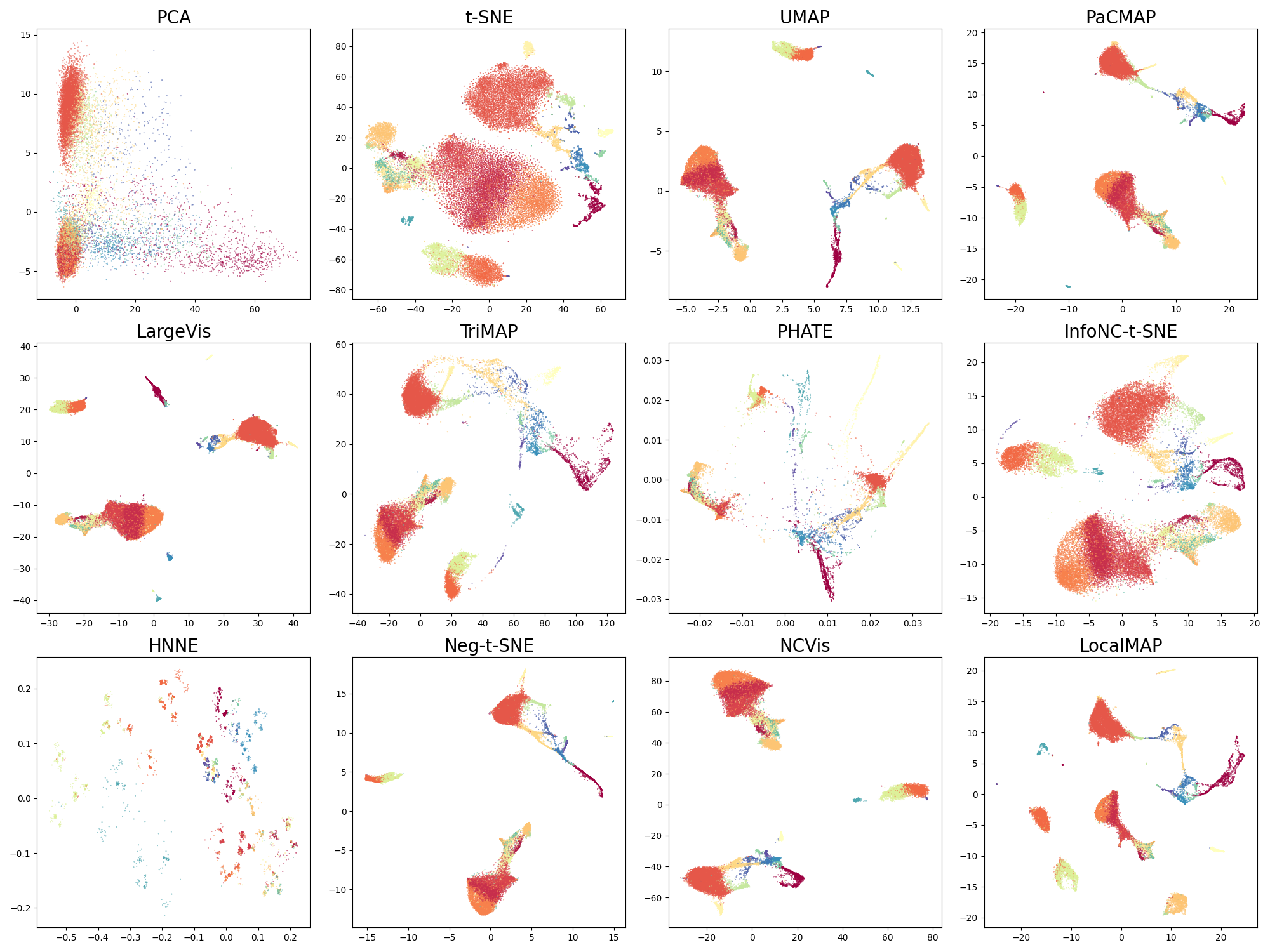}
    \caption{Different DR embeddings for Seurat \citep{stuart2019comprehensive}}
    \label{fig:seurat_embed}
\end{figure*}

\begin{figure*}[ht]
    \centering
    \includegraphics[width=\textwidth]{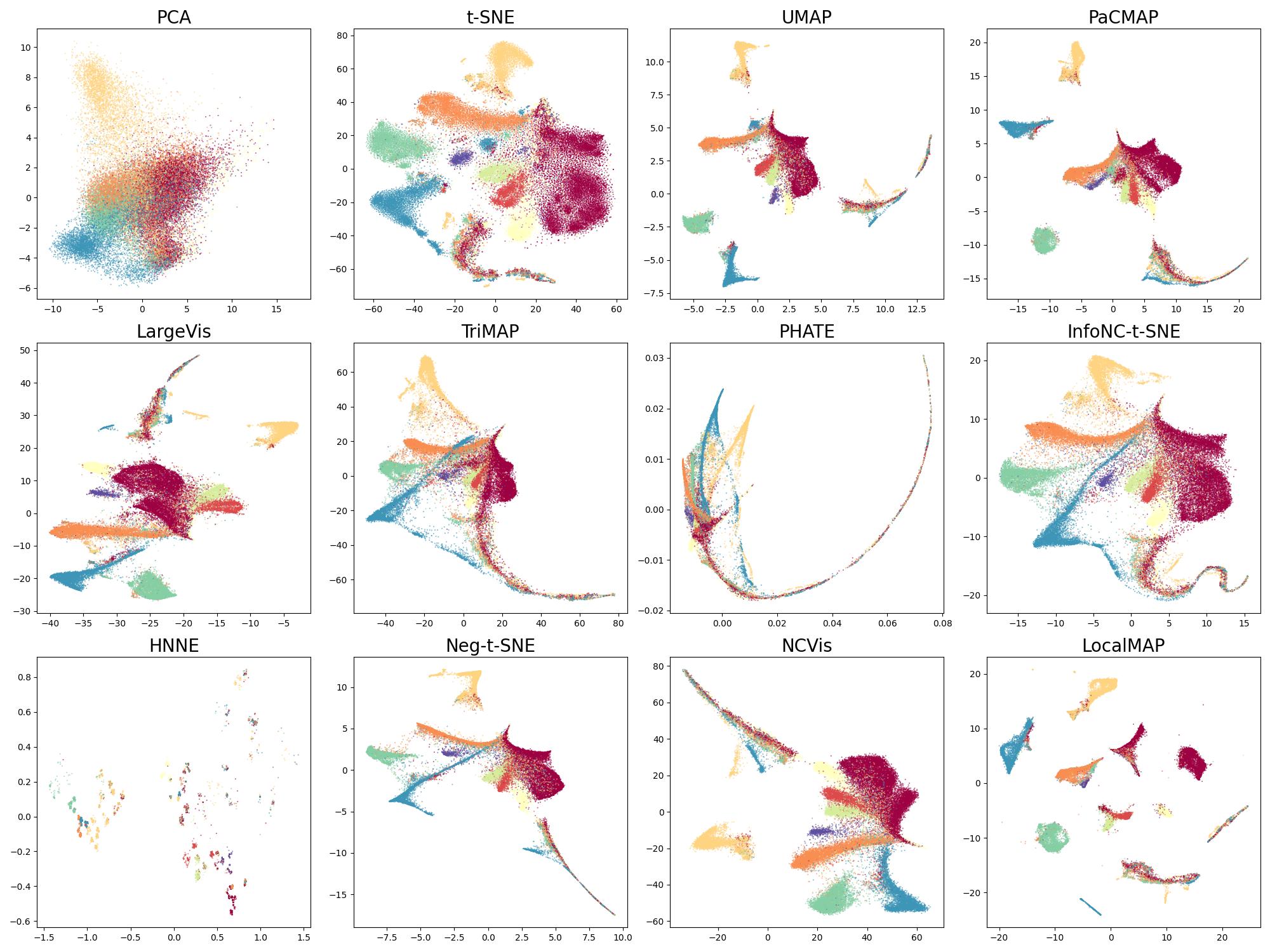}
    \caption{Different DR embeddings for Human Cortex \citep{human_cortex_data}}
    \label{fig:human_cortex_embed}
\end{figure*}

\begin{figure*}[ht]
    \centering
    \includegraphics[width=\textwidth]{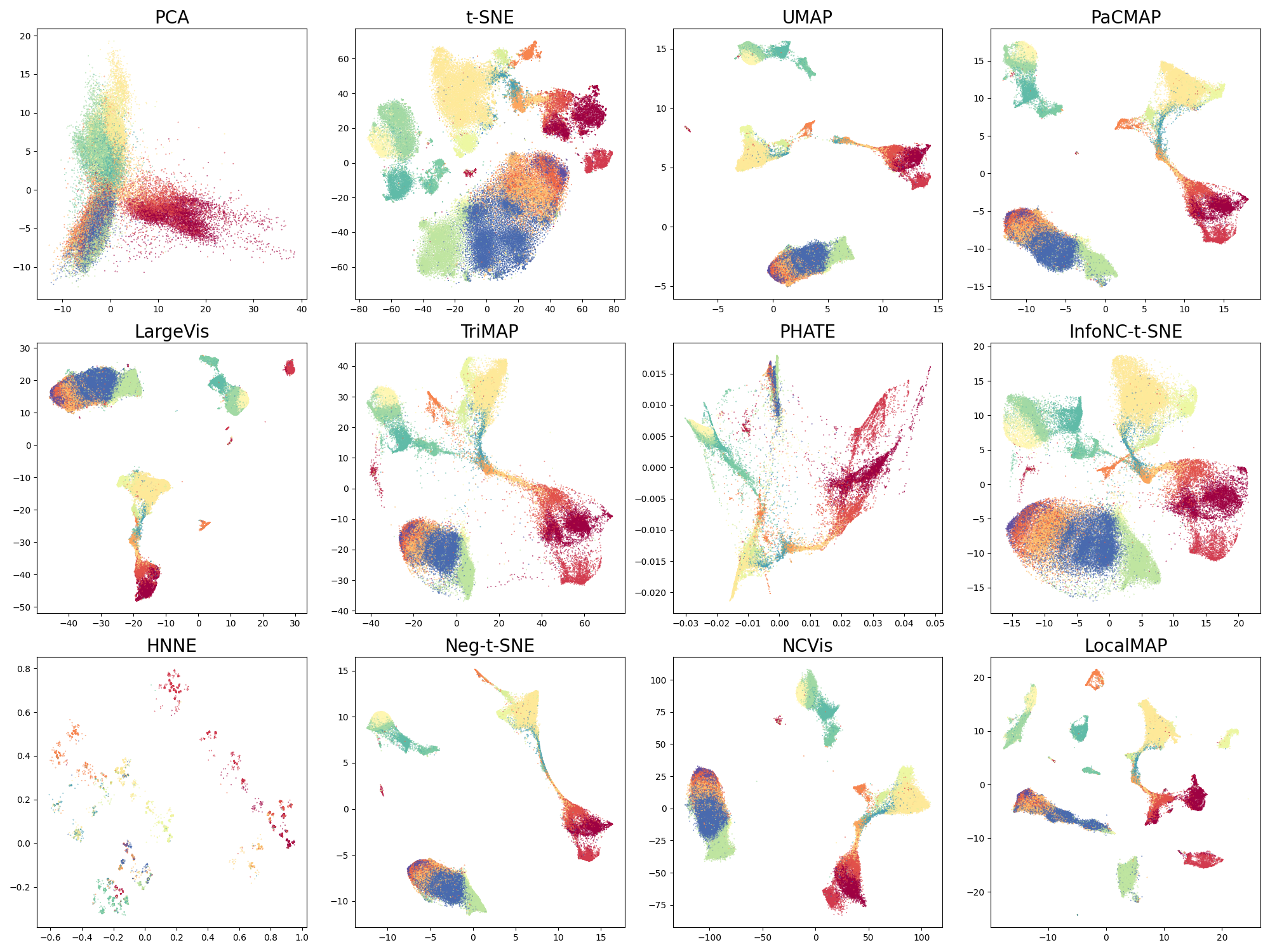}
    \caption{Different DR embeddings for CBMC \citep{neurips2021_data}}
    \label{fig:neurips_embed}
\end{figure*}

\clearpage
\section{Sensitivity Check for Initialization}
\label{app:sensitivity_init}

In this section, we will assess how stable different dimension reduction methods are when points are uniformly randomly initialized in the low-dimensional space. Figures \ref{fig:sensitivtiy_inti_mnist} show the results. LocalMAP is able to reliably separate the 10 clusters for every run. No other method is able to do this for any run -- each has multiple clusters combined that should be separated. t-SNE sometimes has severe flaws in its DR plot for MNIST in that the blue cluster is sometimes broken up; UMAP does this once for one of the red clusters. TriMAP has severe problems with local structure preservation. We have also marked those seriously problematic areas in Figure \ref{fig:sensitivtiy_inti_mnist} with red dashed boxes.

\begin{figure}[ht]
    \centering
    \includegraphics[width=\textwidth]{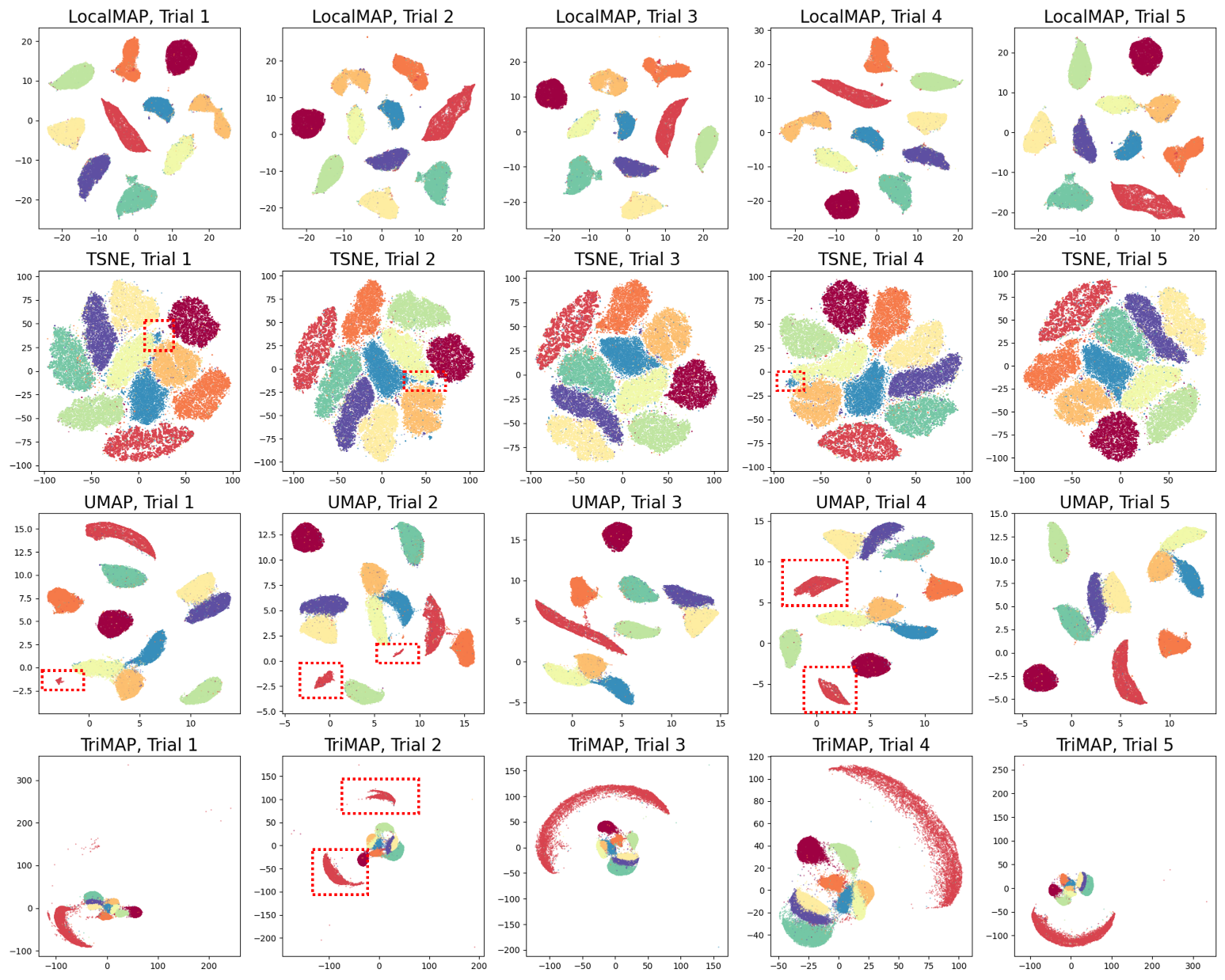}
    \caption{DR embeddings under different initializations for MNIST \citep{lecun2010mnist}; the red dashed boxes represent the broken clusters in the embeddings.}
    \label{fig:sensitivtiy_inti_mnist}
\end{figure}

\clearpage
\section{Comparison with other dimension reduction methods with tuned hyperparameters}
\label{appendix:optimized}

In the main paper, we have shown that LocalMAP can separate clusters better than other approaches. In this section, we compare LocalMAP with its default parameters to t-SNE, UMAP, PHATE and LargeVis with their best parameters. For tuning, we applied grid hyperparameter search, and selected the best hyperparameters among all possible combinations. For t-SNE, we tuned perplexity ([5, 10, 15, 20, 25, 30, 35, 40, 45, 50]) and learning rate ([10, 50, 100, 200, 500, 1000]) based on the suggested range from \citet{gove2022new}. For UMAP, we tuned the number of nearest neighbors ([2, 5, 10, 20, 50, 100, 200]) and the min distance ([0.0, 0.1, 0.25, 0.5, 0.8, 0.99]) based on the official UMAP documentation \citep{UMAP}. For PHATE we used the number of nearest neighbors ([2, 5, 10, 15, 20]) and the decay value ([10, 15, 20, 40, 80, 160]) suggested by the orignal paper\citep{moon2019visualizing}. For LargeVis, we used the range suggested by the original paper \citep{Tang16} and adjusted the perplexity ([10, 50, 100, 200, 500]), the number of times for neighbor propagation (prop) ([1,2,3]), and the weights assigned to negative edges ($\gamma$)([1,3,5,7,9]). The number of neighbors is chosen as three times the perplexity based on the corresponding github document. The weights assigned to negative edges in table \ref{tab:silhouette_highest} show the best parameters we found for each dataset, and table \ref{tab:silhouette_score_optimized} shows the updated comparisons. Based on the table \ref{tab:silhouette_score_optimized}, we can easily observe that LocalMAP can still separate better than the other approaches within most of the datasets. For the COIL20 datasets, we can observe that tuned LargeVis, UMAP, and t-SNE perform slightly better than LocalMAP. However, if we look at the visualizations of these embeddings generated with the tuned hyperparameters in Figure \ref{fig:visualization_comparison_silhouette}, we can easily see that these methods don't provide additional separations for the clusters. Instead, they improve the silhouette score by reducing the intra-cluster distances. Moreover, if we tend to fine-tune the dimension reduction methods to achieve a better performance, it might take more time than using the default hyperparameters, which again proves that LocalMAP are less sensitive to the hyperparameters to achieve a good separation of the clusters.

\begin{table}[ht]
\caption{The best hyperparameters for different dimension reduction methods on different datasets with the highest silhouette score.}
\label{tab:silhouette_highest}
\centering
\renewcommand\arraystretch{1.2}
\begin{tabular}{ccccc}
\hline
 & \textbf{\begin{tabular}[c]{@{}c@{}}t-SNE\\ (perplexity,\\ learning rate)\end{tabular}} & \textbf{\begin{tabular}[c]{@{}c@{}}UMAP\\ (n\_neighbors,\\ min\_dist)\end{tabular}} & \textbf{\begin{tabular}[c]{@{}c@{}}PHATE\\ (k,decay)\end{tabular}} & \textbf{\begin{tabular}[c]{@{}c@{}}LargeVis\\ (perplexity,\\ prop,$\gamma$)\end{tabular}} \\ \hline
\textbf{MNIST} & (50,1000) & (10,0) & (2,15) & (10,2,3) \\
\textbf{FMNIST} & (50,50) & (100,0) & (5,80) & (50,1,3) \\
\textbf{USPS} & (40,500) & (10,10) & (20,160) & (10,2,7) \\
\textbf{COIL20} & (10,500) & (10,0) & (5,160) & (10,1,9) \\
\textbf{20NG} & (40,50) & (100,0.1) & (15,160) & (10,2,1) \\
\textbf{Kang} & (40,500) & (10,0) & (5,20) & (10,1,1) \\
\textbf{Seurat} & (50,200) & (20,0) & (5,160) & (10,1,3) \\
\textbf{Human Cortex} & (45,1000) & (5,0) & (2,160) & (100,2,9) \\
\textbf{CBMC} & (45,500) & (200,0) & (2,80) & (100,1,1) \\ \hline
\end{tabular}
\end{table}

\begin{figure}[ht]
    \centering
    \includegraphics[width=\columnwidth]{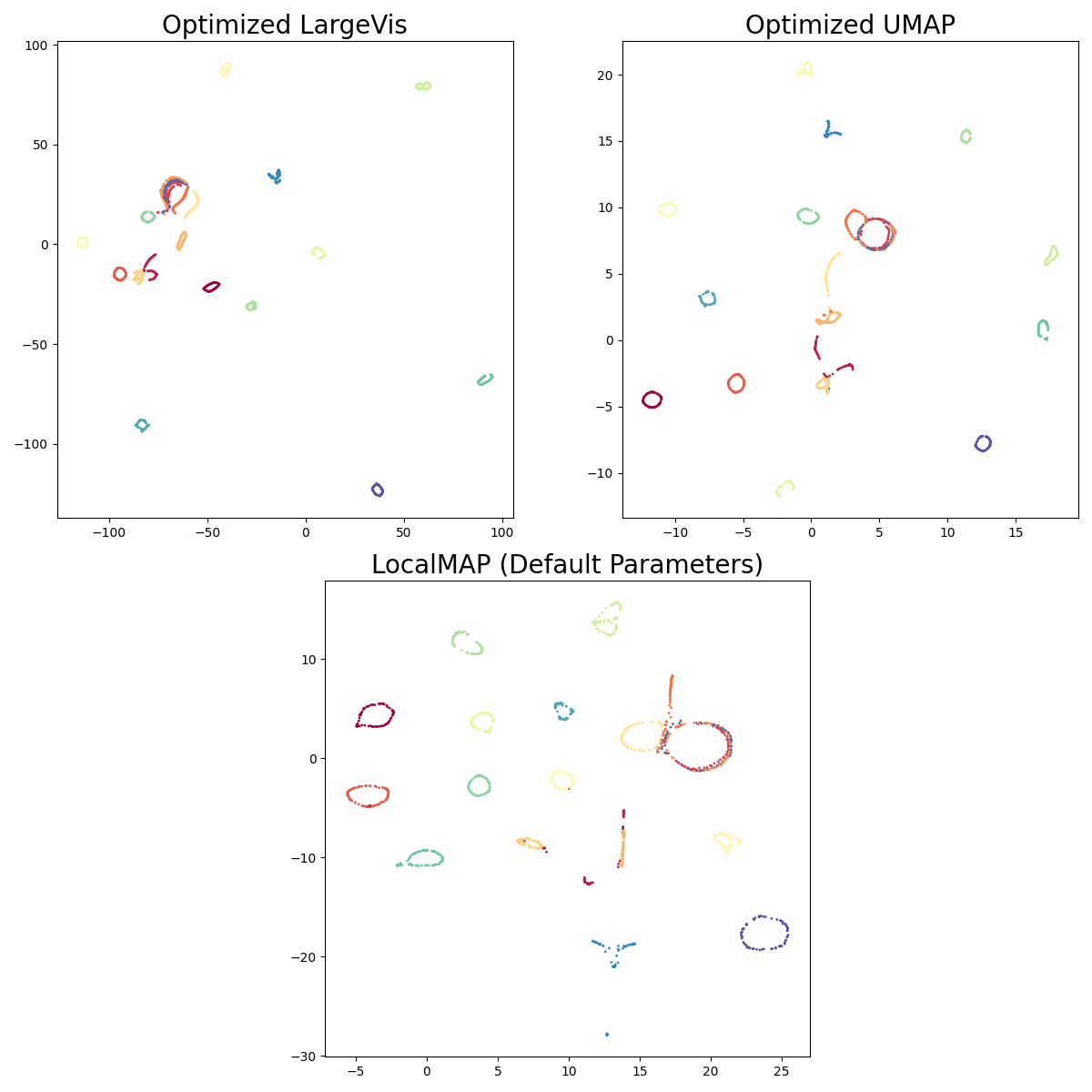}
    \caption{The visualization comparison among LocalMAP, optimized LargeVis and optimzied UMAP.}
    \label{fig:visualization_comparison_silhouette}
\end{figure}

Table \ref{tab:svm_highest} has shown the best hyperparameters for each dimension reduction methods with the highest SVM accuracy within different datasets and table \ref{tab:svm_score_optimized} has shown their corresponding SVM accuracy. Based on the performance of the dimension reduction, we can see that LocalMAP is still comparable with other optimzed DR methods with respect to these scores.

\begin{table}[ht]
\caption{The best hyperparameters for different dimension reduction methods on different datasets with the highest SVM accuracy}
\label{tab:svm_highest}
\centering
\renewcommand\arraystretch{1.2}
\begin{tabular}{ccccc}
\hline
 & \textbf{\begin{tabular}[c]{@{}c@{}}t-SNE\\ (perplexity,\\ learning rate)\end{tabular}} & \textbf{\begin{tabular}[c]{@{}c@{}}UMAP\\ (n\_neighbors,\\ min\_dist)\end{tabular}} & \textbf{\begin{tabular}[c]{@{}c@{}}PHATE\\ (k,decay)\end{tabular}} & \textbf{\begin{tabular}[c]{@{}c@{}}LargeVis\\ (perplexity,\\ prop,$\gamma$)\end{tabular}} \\ \hline
\textbf{MNIST} & (20,200) & (20,0.1) & (5,160) & (10,3,1) \\
\textbf{FMNIST} & (15,1000) & (10,0.1) & (20,80) & (10,3,5) \\
\textbf{USPS} & (20,200) & (10,0) & (20,80) & (10,2,7) \\
\textbf{COIL20} & (5,10) & (5,0.1) & (2,40) & (10,2,9) \\
\textbf{20NG} & (20,1000) & (20,0) & (20,40) & (50,3,3) \\
\textbf{Kang} & (50,10) & (20,0.1) & (10,160) & (100,2,9) \\
\textbf{Seurat} & (50,10) & (50,0.5) & (10,80) & (50,2,5) \\
\textbf{Human Cortex} & (50,10) & (20,0) & (2,80) & (50,2,9) \\
\textbf{CBMC} & (35,50) & (20,0.5) & (2,40) & (100,1,9) \\ \hline
\end{tabular}
\end{table}

\begin{table}[ht]
\centering
\caption{Silhouette scores for different Algorithms. If the method is not labeled as ``Optimized'', then it is using the default hyperparameters. \textbf{Bold} is best, \underline{underline} is not significantly different from best. Each row is an algorithm, each column is a dataset. \text{Red} labels are the ones that have shown significant improvement comparing to LocalMAP}
\label{tab:silhouette_score_optimized}
\renewcommand\arraystretch{1.2}
\scalebox{0.8}{
\begin{tabular}{cccccccccc}
\toprule
 & \textbf{MNIST} & \textbf{FMNIST} & \textbf{USPS} & \textbf{COIL20} & \textbf{20NG} & \textbf{Kang} & \textbf{Seurat} & \begin{tabular}[c]{@{}l@{}}\textbf{Human}\\ \textbf{Cortex} \end{tabular} & \begin{tabular}[c]{@{}l@{}}\textbf{CBMC} \end{tabular} \\
\midrule
\textbf{PCA} & 0.02$\pm$0.00 & -0.03$\pm$0.00 & 0.10$\pm$0.00 & 0.01$\pm$0.00 & -0.19$\pm$0.00 & 0.12$\pm$0.00 & -0.06$\pm$0.00 & -0.08$\pm$0.00 & -0.11$\pm$0.00 \\
\textbf{t-SNE} & 0.35$\pm$0.00 & 0.12$\pm$0.00 & 0.42$\pm$0.00 & 0.41$\pm$0.00 & \underline{-0.11$\pm$0.00} & 0.40$\pm$0.01 & 0.22$\pm$0.01 & 0.11$\pm$0.01 & 0.15$\pm$0.01 \\
\textbf{UMAP} & 0.52$\pm$0.01 & 0.19$\pm$0.01 & 0.53$\pm$0.00 & 0.58$\pm$0.01 & -0.15$\pm$0.01 & 0.51$\pm$0.00 & 0.30$\pm$0.00 & 0.12$\pm$0.02 & \underline{0.22$\pm$0.00} \\
\textbf{PaCMAP} & {0.54$\pm$0.01} & \textbf{0.19$\pm$0.00} & \underline{0.56$\pm$0.00} & 0.51$\pm$0.02 & \underline{-0.11$\pm$0.01} & 0.53$\pm$0.00 & \underline{0.31$\pm$0.00} & \underline{0.13$\pm$0.01} & \underline{0.22$\pm$0.00} \\
\textbf{LargeVis} & 0.49$\pm$0.05 & 0.11$\pm$0.03 & 0.41$\pm$0.12 & 0.38$\pm$0.01 & -0.13$\pm$0.01 & 0.44$\pm$0.01 & 0.25$\pm$0.01 & 0.10$\pm$0.02 & 0.17$\pm$0.00 \\
\textbf{TriMAP} & 0.41$\pm$0.00 & \underline{0.17$\pm$0.00} & 0.48$\pm$0.00 & 0.47$\pm$0.00 & -0.13$\pm$0.00 & \underline{0.55$\pm$0.00} & \textbf{0.32$\pm$0.00} & 0.07$\pm$0.00 & 0.21$\pm$0.00 \\
\textbf{PHATE} & 0.26$\pm$0.02 & 0.11$\pm$0.01 & 0.27$\pm$0.01 & 0.33$\pm$0.00 & -0.21$\pm$0.01 & 0.48$\pm$0.02 & 0.27$\pm$0.01 & -0.09$\pm$0.01 & 0.06$\pm$0.01 \\
\textbf{HNNE} & 0.21$\pm$0.03 & 0.06$\pm$0.04 & 0.23$\pm$0.00 & 0.03$\pm$0.00 & -0.34$\pm$0.03 & 0.39$\pm$0.06 & -0.00$\pm$0.03 & -0.09$\pm$0.06 & 0.12$\pm$0.05 \\
\textbf{Neg-t-SNE} & 0.48$\pm$0.00 & \textbf{0.19$\pm$0.00} & 0.48$\pm$0.00 & 0.44$\pm$0.01 & \underline{-0.11$\pm$0.00} & 0.53$\pm$0.00 & \textbf{0.32$\pm$0.00} & 0.12$\pm$0.00 & \textbf{0.24$\pm$0.00} \\
\textbf{NCVis} & 0.38$\pm$0.02 & \textbf{0.19$\pm$0.00} & 0.44$\pm$0.00 & 0.53$\pm$0.00 & -0.15$\pm$0.00 & 0.51$\pm$0.00 & 0.27$\pm$0.00 & 0.10$\pm$0.00 & 0.20$\pm$0.00 \\
\textbf{InfoNC-t-SNE} & 0.33$\pm$0.00 & 0.13$\pm$0.00 & 0.37$\pm$0.00 & 0.43$\pm$0.01 & \underline{-0.11$\pm$0.00} & 0.46$\pm$0.00 & 0.26$\pm$0.00 & 0.10$\pm$0.00 & 0.21$\pm$0.00 \\
\textbf{Optimized LargeVis} & \underline{0.56$\pm$0.01} & \textbf{0.19$\pm$0.00} & 0.55$\pm$0.03 & \textcolor{red}{\textbf{0.65$\pm$0.03}} & -0.12$\pm$0.01 & 0.47$\pm$0.01 & 0.27$\pm$0.02 & 0.12$\pm$0.00 & \underline{0.22$\pm$0.00}\\
\textbf{Optimized PHATE} & 0.34$\pm$0.02 & 0.12$\pm$0.00 & 0.55$\pm$0.00 & 0.33$\pm$0.00 & -0.18$\pm$0.01 & 0.51$\pm$0.01 & 0.27$\pm$0.01 & -0.01$\pm$0.00 & 0.10$\pm$0.02\\
\textbf{Optimized t-SNE} & 0.39$\pm$0.00 & 0.14$\pm$0.00 & 0.43$\pm$0.01 & 0.51$\pm$0.00 & \underline{-0.11$\pm$0.00} & 0.43$\pm$0.02 & 0.23$\pm$0.00 & 0.13$\pm$0.01 & 0.18$\pm$0.01\\
\textbf{Optimized UMAP} & \textbf{0.58$\pm$0.00} & \textbf{0.19$\pm$0.00} & \textbf{0.60$\pm$0.00} & \textcolor{red}{\underline{0.63$\pm$0.00}} & -0.13$\pm$0.01 & 0.54$\pm$0.01 & \underline{0.31$\pm$0.00} 
& \textbf{0.14$\pm$0.00} & \underline{0.22$\pm$0.00}\\
\textbf{LocalMAP} & \textbf{0.58$\pm$0.00} & \textbf{0.19$\pm$0.00} & \textbf{0.60$\pm$0.00} & 0.56$\pm$0.01 & \textbf{-0.10$\pm$0.00} & \textbf{0.60$\pm$0.00} & \textbf{0.32$\pm$0.00} & \textbf{0.14$\pm$0.00} & \underline{0.22$\pm$0.00} \\
\bottomrule
\end{tabular}}
\end{table}

\begin{table}[ht]
\caption{SVM Score for Different Algorithms. If the method is not labeled as ``Optimized'', then it is using the default hyperparameters. \textbf{Bold} is best, \underline{underline} is not significantly different from best (with only 1\% difference). Each row is an algorithm, each column is a dataset.}
\label{tab:svm_score_optimized}
\centering
\renewcommand\arraystretch{1.2}
\scalebox{0.8}{
\begin{tabular}{cccccccccc}
\toprule
 & \textbf{MNIST} & \textbf{FMNIST} & \textbf{USPS} & \textbf{COIL20} & \textbf{20NG} & \textbf{Kang} & \textbf{Seurat} & \begin{tabular}[c]{@{}l@{}}\textbf{Human}\\ \textbf{Cortex} \end{tabular} & \begin{tabular}[c]{@{}l@{}}\textbf{CBMC} \end{tabular} \\
\midrule
\textbf{PCA} & 0.47$\pm$0.00 & 0.55$\pm$0.00 & 0.56$\pm$0.00 & 0.66$\pm$0.00 & 0.15$\pm$0.00 & 0.73$\pm$0.00 & 0.46$\pm$0.00 & 0.57$\pm$0.00 & 0.44$\pm$0.00 \\
\textbf{t-SNE} & \textbf{0.97$\pm$0.00} & \underline{0.74$\pm$0.00} & \textbf{0.96$\pm$0.00} & 0.85$\pm$0.01 & 0.45$\pm$0.01 & \underline{0.95$\pm$0.00} & \underline{0.84$\pm$0.00} & \textbf{0.82$\pm$0.00} & 0.82$\pm$0.00 \\
\textbf{UMAP} & \textbf{0.97$\pm$0.00} & \underline{0.74$\pm$0.01} & \underline{0.95$\pm$0.00} & 0.82$\pm$0.01 & 0.44$\pm$0.01 & \underline{0.95$\pm$0.00} & 0.83$\pm$0.00 & \underline{0.81$\pm$0.00} & \underline{0.82$\pm$0.00} \\
\textbf{PaCMAP} & \textbf{0.97$\pm$0.00} & \underline{0.74$\pm$0.00} & \underline{0.95$\pm$0.00} & 0.83$\pm$0.01 & \underline{0.46$\pm$0.01} & \underline{0.95$\pm$0.00} & \textbf{0.85$\pm$0.00} & \underline{0.81$\pm$0.00} & \textbf{0.83$\pm$0.00} \\
\textbf{LargeVis} & \underline{0.96$\pm$0.00} & \underline{0.74$\pm$0.01} & 0.92$\pm$0.06 & 0.80$\pm$0.02 & \textbf{0.47$\pm$0.00} & \underline{0.95$\pm$0.00} & \underline{0.84$\pm$0.00} & \textbf{0.82$\pm$0.00} &\underline{0.82$\pm$0.00} \\
\textbf{TriMAP} & \underline{0.96$\pm$0.00} & 0.73$\pm$0.00 & \underline{0.95$\pm$0.00} & 0.77$\pm$0.01 & 0.42$\pm$0.01 & \underline{0.95$\pm$0.00} & \underline{0.84$\pm$0.00} & 0.79$\pm$0.00 & \underline{0.82$\pm$0.00} \\
\textbf{PHATE} & 0.86$\pm$0.02 & 0.66$\pm$0.01 & 0.86$\pm$0.01 & 0.84$\pm$0.00 & 0.33$\pm$0.01 & 0.92$\pm$0.00 & 0.77$\pm$0.00 & 0.70$\pm$0.01 & 0.72$\pm$0.01 \\
\textbf{HNNE} & 0.84$\pm$0.03 & 0.68$\pm$0.01 & 0.82$\pm$0.00 & 0.63$\pm$0.00 & 0.24$\pm$0.05 & 0.90$\pm$0.01 & 0.74$\pm$0.01 & 0.68$\pm$0.03 & 0.73$\pm$0.04 \\
\textbf{Neg-t-SNE} & \underline{0.96$\pm$0.00} & \underline{0.74$\pm$0.00} & 0.93$\pm$0.00 & 0.81$\pm$0.01 & 0.43$\pm$0.01 & \underline{0.95$\pm$0.00} & \underline{0.84$\pm$0.00} & \underline{0.81$\pm$0.00} & \underline{0.82$\pm$0.00} \\
\textbf{NCVis} & 0.94$\pm$0.01 & 0.73$\pm$0.00 & 0.92$\pm$0.00 & 0.79$\pm$0.00 & 0.36$\pm$0.01 & 0.94$\pm$0.00 & 0.83$\pm$0.00 & \textbf{0.82$\pm$0.00} & \underline{0.82$\pm$0.00} \\
\textbf{InfoNC-t-SNE} & \underline{0.96$\pm$0.00} & \underline{0.74$\pm$0.00} & 0.93$\pm$0.00 & 0.82$\pm$0.01 & 0.42$\pm$0.00 & \underline{0.95$\pm$0.00} & \textbf{0.85$\pm$0.00} &\underline{0.81$\pm$0.00} & \textbf{0.83$\pm$0.00} \\
\textbf{LocalMAP} & \textbf{0.97$\pm$0.00} & \textbf{0.75$\pm$0.00} & \textbf{0.96$\pm$0.00} & 0.83$\pm$0.01 & \underline{0.46$\pm$0.01} & \textbf{0.96$\pm$0.00} & \underline{0.84$\pm$0.00} & \underline{0.81$\pm$0.00} & \underline{0.82$\pm$0.00} \\
\textbf{Optimized LargeVis} &  \textbf{0.97$\pm$0.00}  & \underline{0.74$\pm$0.01} & \textbf{0.96$\pm$0.00} & 0.85$\pm$0.01 & \underline{0.46$\pm$0.00} & \underline{0.95$\pm$0.00} & \underline{0.84$\pm$0.00} & \textbf{0.82$\pm$0.01} & \underline{0.82$\pm$0.00}\\
\textbf{Optimized t-SNE} & \textbf{0.97$\pm$0.00} & \textbf{0.75$\pm$0.00} & \textbf{0.96$\pm$0.00} & \underline{0.92$\pm$0.02} & \underline{0.46$\pm$0.00} & \underline{0.95$\pm$0.00} & \underline{0.84$\pm$0.01} & \textbf{0.82$\pm$0.00} & \underline{0.82$\pm$0.01} \\
\textbf{Optimized UMAP} & \textbf{0.97$\pm$0.00} & \textbf{0.75$\pm$0.00} & \textbf{0.96$\pm$0.00} & 0.89$\pm$0.00 & 0.44$\pm$0.01 & \underline{0.95$\pm$0.00} & \underline{0.84$\pm$0.00} & \textbf{0.82$\pm$0.00} & \underline{0.82$\pm$0.00}\\
\textbf{Optimized PHATE} & 0.89$\pm$0.01 & 0.68$\pm$0.02 & 0.86$\pm$0.00 & \textbf{0.93$\pm$0.01} & 0.37$\pm$0.00 & 0.92$\pm$0.01 & 0.77$\pm$0.00 & 0.71$\pm$0.01 & 0.72$\pm$0.00 \\
\textbf{LocalMAP} & \textbf{0.97$\pm$0.00} &\textbf{0.75$\pm$0.00} & \textbf{0.96$\pm$0.00} & 0.83$\pm$0.01 & \underline{0.46$\pm$0.01} & \textbf{0.96$\pm$0.00} & \underline{0.84$\pm$0.00} & \underline{0.81$\pm$0.00} & \underline{0.82$\pm$0.00} \\
\bottomrule
\end{tabular}}
\end{table}

\clearpage
\section{Scalability of LocalMAP under Large Datasets}

In this section, we have added two large single-cell datasets with more than 1 million cells within each dataset to show the scalability of our model. The data description of the extended dataset within our model has already been shown in Table \ref{tab:large_tab}. The biological data sets are processed according to the same method mentioned in Section \ref{sec:experiment}, and the detailed embeddings using PaCMAP and LocalMAP are shown in Figure \ref{fig:large_scale}, which proves that LocalMAP shows good separation even in large-scale settings.

\begin{table}[ht]
\centering
\caption{Data Description with Large Scale Dataset over 1 million samples}
\label{tab:large_tab}
\begin{tabular}{ccc}
\hline
\textbf{Dataset} & \textbf{\# of samples} & \textbf{\# of dimensions} \\ \hline
PBMC 1M\cite{perez2022single} & 1,263,676 & 1,000 \\
AIDA\cite{kock2024single} & 1,058,909 & 1,000 \\\hline
\end{tabular}
\end{table}

\begin{figure}[ht]
    \centering
    \includegraphics[width=0.7\linewidth]{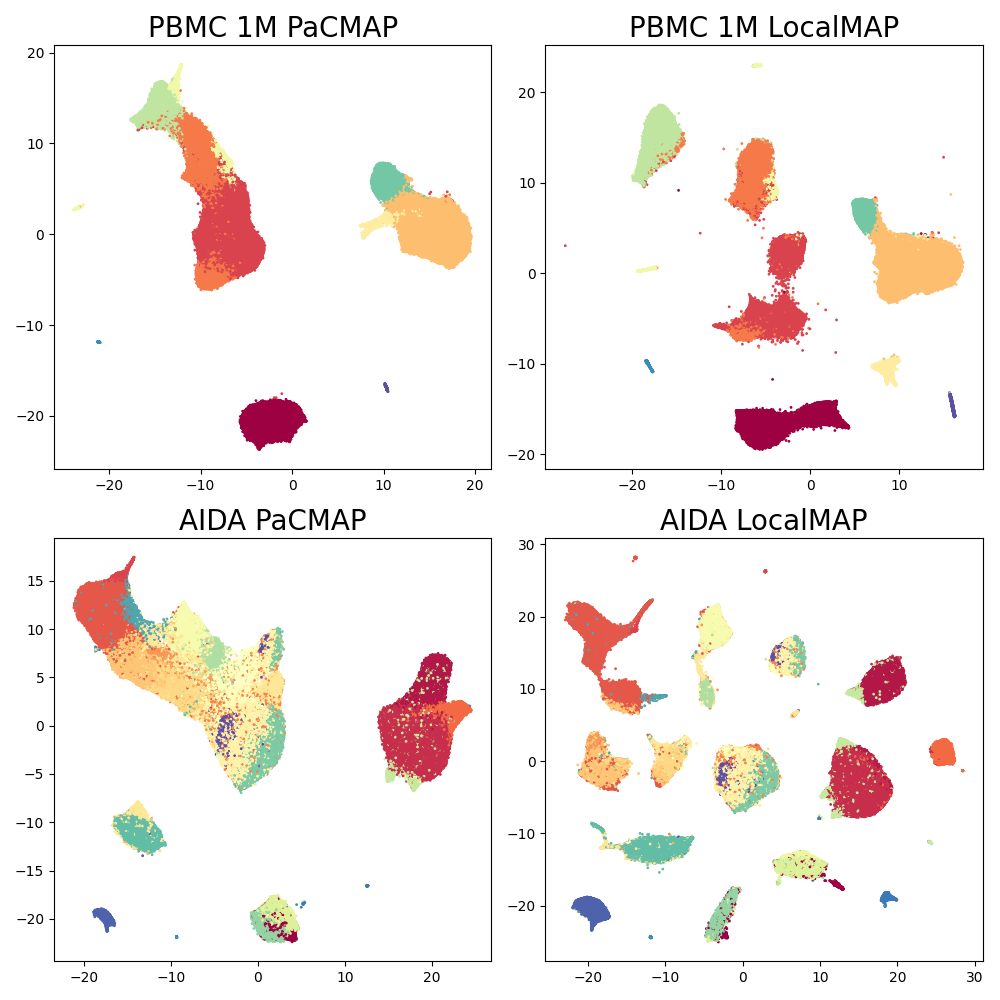}
    \caption{The performance of PaCMAP and LocalMAP under large scale settings}
    \label{fig:large_scale}
\end{figure}

\end{document}